\documentclass{article}

\usepackage{microtype}
\usepackage{graphicx}
\usepackage{subcaption}
\usepackage{booktabs} % for professional tables

\usepackage{hyperref}

\usepackage[preprint]{icml2026}

\usepackage[utf8]{inputenc} % allow utf-8 input
\usepackage[T1]{fontenc}    % use 8-bit T1 fonts
\usepackage{url}            % simple URL typesetting
\usepackage{amsfonts}       % blackboard math symbols
\usepackage{nicefrac}       % compact symbols for 1/2, etc.
\usepackage{xcolor}

\usepackage{amsthm}
\usepackage{amsmath,amssymb,bm}

\usepackage{algorithm}
\usepackage{algorithmic}

\renewcommand{\algorithmicrequire}{\textbf{Input:}}
\renewcommand{\algorithmicensure}{\textbf{Output:}}

\newcommand{\algorithmicHyper}{\textbf{Hyperparameters:}}
\newcommand{\HYPER}{\item[\algorithmicHyper]}

\newcommand{\AlgComment}[1]{\STATE \textit{$\triangleright$ #1}}

\theoremstyle{plain} 
\newtheorem{lemma}{Lemma}
\newtheorem{theorem}{Theorem}
\newtheorem{corollary}{Corollary}
\newtheorem{proposition}{Proposition}

\theoremstyle{definition}

\newtheorem{example}{Example}

\newcommand{\Null}{\operatorname{Null}}

\icmltitlerunning{Local EGOP for Continuous Index Learning}

\begin{document}

\twocolumn[
\icmltitle{Local EGOP for Continuous Index Learning}

\begin{icmlauthorlist}
\icmlauthor{Alex Kokot}{uwstat}
\icmlauthor{Anand Hemmady}{uwbio}
\icmlauthor{Vydhourie Thiyageswaran}{uwstat}
\icmlauthor{Marina Meila}{uwaterloo}
\end{icmlauthorlist}

\icmlaffiliation{uwstat}{Department of Statistics, University of Washington}
\icmlaffiliation{uwbio}{Department of Biostatistics, University of Washington}
\icmlaffiliation{uwaterloo}{School of Computer Science, University of Waterloo}

\icmlcorrespondingauthor{Alex Kokot}{akokot@uw.edu} % Update with actual email

\icmlkeywords{Machine Learning, ICML, EGOP, Continuous Index Learning}

\vskip 0.3in
]

% This command actually creates the footnote in the first column listing the affiliations.
\printAffiliationsAndNotice{} 

\begin{abstract}
  We introduce the setting of continuous index learning, in which a function of many variables varies only along a small number of directions at each point. For efficient estimation, it is beneficial for a learning algorithm to adapt, near each point $x$, to the subspace that captures the local variability of the function $f$.
  We pose this task as kernel adaptation along a manifold with noise, and introduce Local EGOP learning, a recursive algorithm that utilizes the Expected Gradient Outer Product (EGOP) quadratic form as both a metric and inverse-covariance of our target distribution. We prove that Local EGOP learning adapts to the regularity of the function of interest, showing that under a supervised noisy manifold hypothesis, intrinsic dimensional learning rates are achieved for arbitrarily high-dimensional noise. Empirically, we compare our algorithm to the feature learning capabilities of deep learning. Additionally, we demonstrate improved regression quality compared to two-layer neural networks in the continuous single-index setting.
\end{abstract}

\section{Problem and approach}
\label{sec:intro}
\newcommand{\M}{\mathcal{M}}   

    % Some objectives of kernel engineering include the design of specialized kernels suited to particular data structures (\citep{odone2005building, kondor2003kernel, joachims1998text, chapelle1999support, barla2002hausdorff, vishwanathan2010graph}, etc.), statistical principles (\citep{genton2001classes, scholkopf1997prior, osborne2010bayesian}, etc.),  and problems of interest (\citet{gong2024supervised, kokot2025coreset}, etc.). 
    % In regression settings, a classical approach is local feature learning, in which kernels are augmented by differential information at points of interest (\citep{schmid1997local, wallraven2003recognition, lowe1999object}).
    % Earlier nonparametric methods developed a similar framework, with datasets being recursively partitioned to improve the quality of local fits \citep{heise1971multivariate, breiman1976general, friedman2006tree}.

\begin{figure}[tb]
    \centering
    \includegraphics[width=1\linewidth]{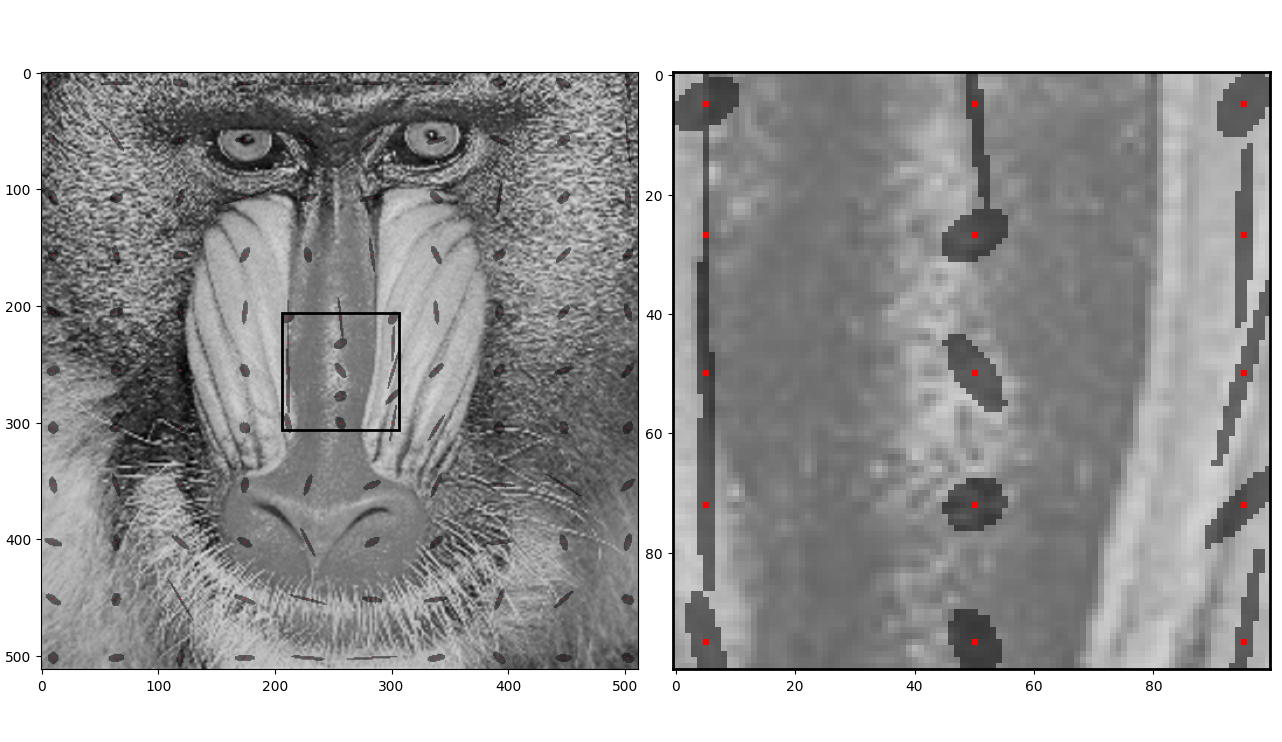}
    \caption{Localizations from Local EGOP Learning (Algorithm 1) centered at each of the highlighted points (in red) trained on a grayscale image of a mandrill \citet{bc9m-f507-21}. Here the inputs $x$ are the pixel location and $f(x)$ is the grayscale value of the image. For visualization purposes the procedure was stopped early, and the 125 highest weight pixels are highlighted at each point. On the right the image is magnified to the highlighted region boxed-off on the left.}
    \label{fig: Mandrill}
\end{figure}

    % A modern incarnation of kernel engineering is {\em multi-index learning}, particularly in the case of neural networks (\citep{mousavi2022neural, boix2023transformers, damian2023smoothing}, etc.). This literature aims to show that the desirable properties of kernel engineering, such as data adaptivity and dimension reduction, are captured implicitly by certain machine learning models. Much work has been done in the {\em single-index} case, where the goal is to regress labels $Y$ on features $X$ when $f(x) := \mathbb{E}[Y|X=x]$ depends on $x$ solely through their evaluation in a fixed direction $v,$ via the quantity $v^T x$. It has been shown that two-layer neural networks not only learn this dependence, but also do so efficiently, leading to rapid increases in prediction quality (\citep{bietti2022learning, abbe2024mergedstaircasepropertynecessarynearly, lee2024neural}, etc.). Recent results indicate that these same benefits carry over to the multi-index setting for models trained via SGD \citep{damian2023smoothing, arnaboldi2024repetita}.

In a variety of contexts, including speech, images, and molecular dynamics, high-dimensional data often lie near a non-linear but low-dimensional manifold.
Motivated by this phenomenon, we formulate the \textit{supervised noisy manifold hypothesis (SNMH)}, which posits that a target function $f(x)$ varies along, but not orthogonally to, such a manifold $\M$. We call the problem of estimating $f$ {\em continuous index learning (CIL)}, analogous to the well-studied problem of {\em multi-index learning (MIL)}, which assumes that $\M$ is a linear subspace.

% We call this learning problem {\em continuous index
%   learning (CIL)}, by analogy with the well-studied {\em multi-index
%   learning (MIL)}, which assumes that $\M$ is a linear subspace.

% Motivated by this phenomenon, we introduce the \textit{supervised noisy manifold hypothesis (SNMH)}, which posits that the $D$-dimensional features $X$ are concentrated about a $d$-dimensional submanifold $\M$. 
% We formulate the problem of learning a function $f(x)$ that varies along, but not orthogonally to, $\M$, under the SNMH. We call this learning problem {\em continuous index
%   learning (CIL)}, by analogy with the well-studied {\em multi-index
%   learning (MIL)}, which assumes that $\M$ is a linear subspace.

% High-dimensional data, such as speech, images, and protein conformations, often lie near a non-linear but low-dimensional submanifold, denoted by $\M$. \comment{Deep learning algorithms are able to implicitly recover $\M$, however no guarantees w.r.t. this exist yet.} In this
% paper, we formulate the problem of learning a function $f(x)$ that varies along but not orthogonally to $\M$ (whose degree we denote by $d$), under the assumption that the features
% $X$ are concentrated around $\M$ (the \textit{supervised noisy manifold
%   hypothesis (SNMH)}). We call this learning problem {\em continuous index
%   learning (CIL)}, by analogy with the well-studied {\em multi-index
%   learning (MIL)}, which assumes that $\M$ is a linear subspace.

Given observations $(x_i,y_i)$, $i\in [n]$, multi-index learning is the task of estimating a function $f(x) := \mathbb{E}[Y|X=x]$ that only depends on
$x$ through composition with a low-rank matrix $V$, i.e. $f(x) = g(Vx).$ In
CIL, we allow $V_x$ to be a continuously varying affine map, i.e. $f(x)
=g(V_x x)$, and $\nabla V_x \eta = 0$ for all $\eta$ such that $V_x(x+\eta) = V_x(x)$. Due to the SNMH, $f(x) = g( \pi (x))$ where $\pi(x) :=
\operatorname{argmin}_{q \in \mathcal{M}} \|q-x\|$ denotes the nearest point
projection onto $\mathcal{M}$. For such data we call the space normal to $\M$, denoted $U_{\pi(x)}$, the {\em uninformative subspace}, as $f$ does not vary along these directions. We refer to the
tangent directions, denoted $V_x = U_{\pi(x)}^\perp$, as the {\em
  informative subspace}.
\begin{figure}[tb]
    \centering
    \includegraphics[width=0.5\linewidth]{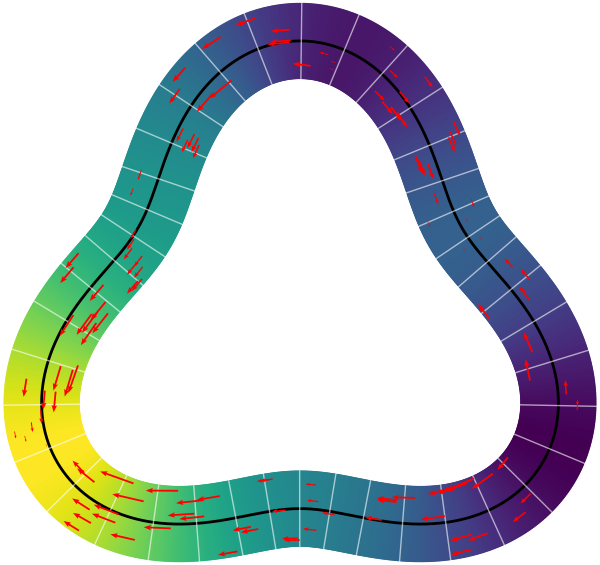}
    \caption{Supervised noisy manifold data in $\mathbb{R}^2$. In a neighborhood about a closed curve  we overlay a heatmap of a function invariant to normal displacement. The curve is displayed in black, the normal spaces are highlighted in white, and red gradients are displayed at randomly sampled points, with lengths proportionate to the gradient magnitude.}
    \label{fig:noisy manifold}
\end{figure}

We approach estimating $f(x)$ via kernel smoothing, letting
\begin{equation} \label{eq:fhat}
\hat{f}(x) = \sum_i w_i Y_i,\;\text{with }w_i \propto k_h(x,X_i)
\end{equation}
for a kernel $k$ and bandwidth $h$. For convenience, we use a Gaussian kernel for $k$, although modifications for the general setting are presented in Appendix \ref{app: general kernels}.  

We propose that the $k_h$ be {\em anisotropic}, emulating the local
multi-index structure of $f$ through the Mahalanobis metrization
$k_h(x,X_i)=k(\|M^{1/2}(x - X_i)\|/\sqrt{2h})$.
Specifically, if $M$
can be selected to degenerate along the normal $U_{\pi(x)}$, then the
weights $w_i$ will decay primarily along the informative tangent
direction $V_x$, leading to estimation rates scaling with the
dimension $d$ of the intrinsic manifold $\M$.

We then develop a
recursive procedure we call \textit{Local EGOP Learning}, which
iteratively refines the metric $M$ by pooling coarse differential
information estimated from the observed data around a target point $x^*,$ without a priori knowledge of the manifold $\mathcal{M}$.

Our method draws inspiration from recent works such as \citet{radhakrishnan2022mechanism, beaglehole2023mechanism,radhakrishnan2025linear, zhu2025iteratively}; more background can be found in Section \ref{sec:related}.
Central in this line of research is the expected gradient outer product (EGOP), $\mathcal{L}(\mu):= \int \nabla f \nabla f^T\, d\mu$. Its empirical counterpart is the average gradient outer product (AGOP), $\hat{\mathcal{L}}(P_n) := \mathcal{L}_{ \hat f}(P_n):= \frac{1}{n} \sum_{i=1}^n  \nabla \hat f(x_i)  \nabla \hat f^T(x_i)$, with $P_n$ the empirical distribution, $\hat f$ an estimator of $f$. 
In this paper, we focus on the problem of point estimation, where $f$
is estimated at a specified query point $x^*$.
Our key insight is to iteratively
\textit{localize} the EGOP (or AGOP) around the point of interest $x^*$. At a given iteration $t$, we set $w_i^{(t)} \propto k(\|\hat{M}_t^{1/2}(x^*-X_i)\|/ \sqrt{2 h_{t+1}})$, use $\hat{f}$ given by \eqref{eq:fhat}, and compute
\begin{align*}
  \hat{\mathcal{L}}_t := \hat{\mathcal{L}} \left(\sum_{i=1}^n w_i^{(t)} \delta_{x_i}\right)\;=
 \sum_{i=1}^n w_i^{(t)}  \nabla \hat f(x_i) \nabla \hat f(x_i)^T.
\end{align*}

That is, rather than averaging
the gradient outer products over the full dataset, we \textit{weigh}
them with respect to the same kernel we utilize for regression.
In the subsequent iteration, we update the metric $\hat{M}_{t+1}$ using $\hat{\mathcal{L}}_t$, thus steering the kernel with this differential information.
 This recursion is at the core of Algorithm \ref{alg:LEGOP} described in Section \ref{sec:algo},

In Section \ref{sec:theory}, we analyze this recurrence through its continuum counterparts, $\mu_{t+1} := N(x^*, h_{t+1} M_t^{-1})$, $M_t := \beta \mathcal{L}(\mu_t) + (1-\beta) \mathcal{L}(\mu_{t-1})$. In Section \ref{sec:expe}, we numerically validate these results, and compare the performance of our proposed estimator to various neural network architectures, while Section \ref{sec:discussion} indicates avenues for future work.
Our main contributions are:
\begin{enumerate}\setlength{\partopsep}{0em}\setlength{\parskip}{0em}
\item Introducing the setting of continuous index learning, generalizing multi-index learning to latent manifold structures.
\item Developing the Local EGOP Learning Algorithm and the localization of the EGOP. 
\item Analyzing the recurrence and the resulting intrinsic dimensional learning rate of Local EGOP Learning under the SNMH (Sections \ref{sec: intuition} and \ref{sec: guarantees}).
\item Establishing a novel
  approach for the study of recursive kernel learning such as
  \citet{radhakrishnan2022mechanism, radhakrishnan2025linear,
    zhu2025iteratively}.
\item  Clarifying the role of the EGOP in the minimization of the mean squared error via a new functional framework.
\end{enumerate}

\begin{figure}[tb]
    \centering
    \begin{subfigure}[b]{0.65\linewidth}
    \centering
    \includegraphics[width=\linewidth]{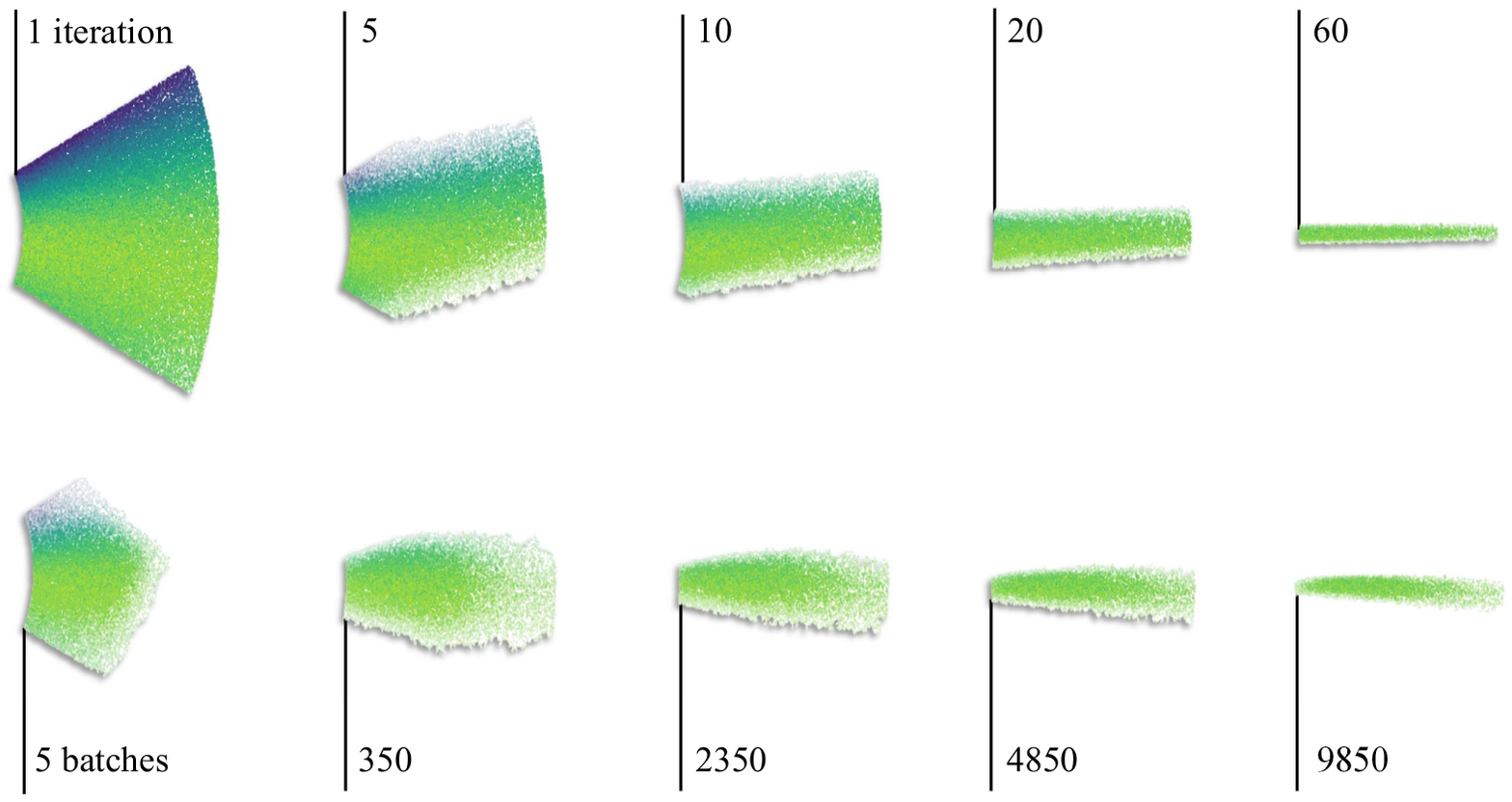}
    \subcaption{Localizations $\mu_i$ from Local EGOP Learning (top) and a deep neural network feature embedding (bottom).}
    \label{fig:agop_descent}
\end{subfigure}
\hfill
\begin{subfigure}[b]{0.32\linewidth}
    \centering
    \includegraphics[angle=90,origin=c,width=\linewidth]{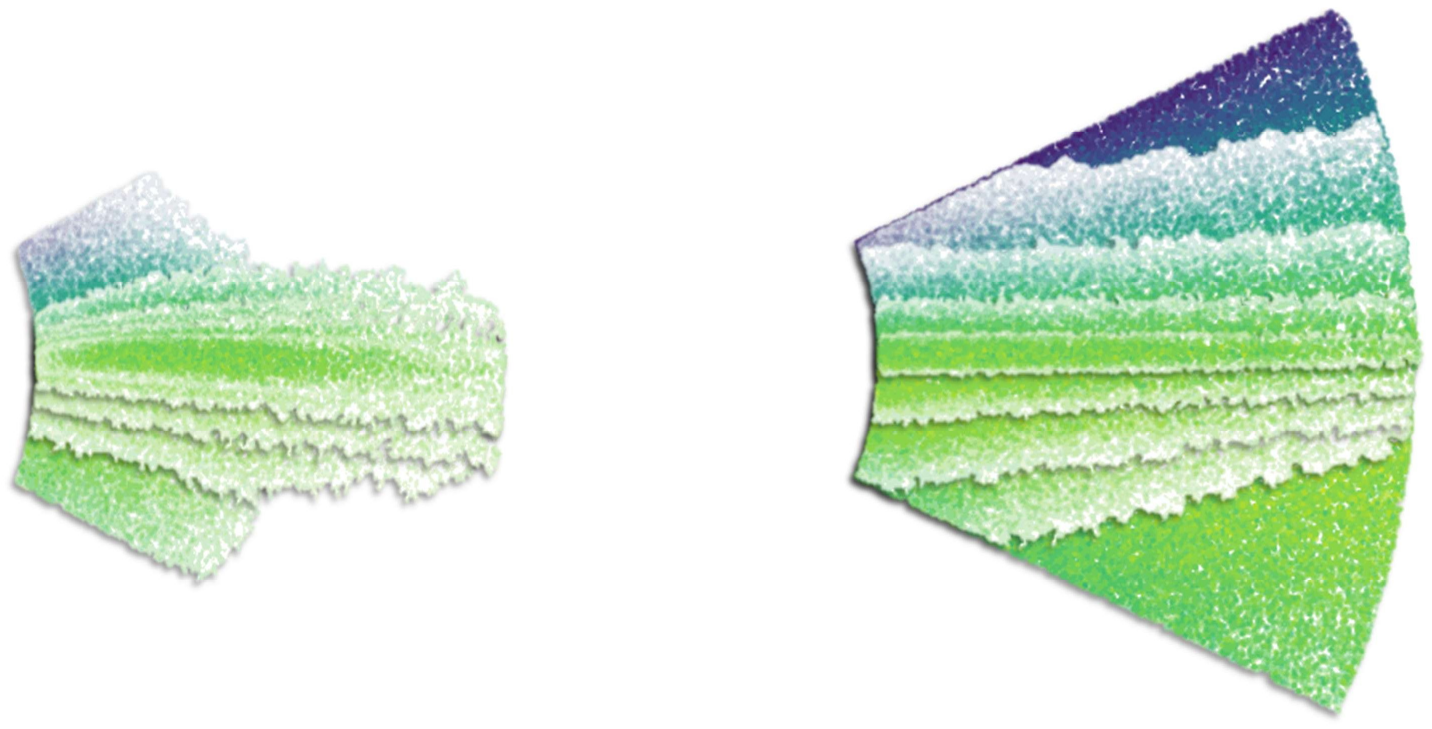}
    \subcaption{Layered EGOP (top) and neural network (bottom) localizations.}
    \label{fig:deep_feet}
\end{subfigure}
    \caption{Local EGOP Learning and a deep transformer architecture applied to data sampled from an annulus. Labels are generated with no dependence on the radius in the parameterization, with these values overlaid as a heatmap. About the point $x^* = (1,0)$, points are displayed with opacity $w \propto \exp(-d(x^*,x_i)^2)$, where $d(\cdot, \cdot)$ is the AGOP Mahalanobis distance (top) and transformer feature embedding (bottom). Both regions are displayed after progressively many iterations/training batches.}
    \label{fig:agop_combined}
\end{figure}

\section{Background and Related Work}
\label{sec:related}
 Kernel methods are a powerful mechanism for studying machine learning algorithms. Many algorithms have leveraged the corresponding RKHS structure for efficient estimation of sufficiently regular functions \citep{wainwright2019high}, and functionals \citep{rao2014nonparametric}. 
    Further, many popular learning algorithms, such as certain neural network architectures and random forests, have been shown to asymptotically correspond to carefully chosen kernels \citep{jacot2018neural, scornet2016random}.
    Thus, the process by which investigators select kernels tailored to problems of interest, known as kernel engineering, is of central importance \citep{belkin2018understand}.
    Beyond allowing for potential efficiency gains, kernel engineering closely emulates the feature learning properties of deep neural networks.

Some objectives of kernel engineering include the design of specialized kernels suited to particular data structures (\citet{odone2005building, kondor2003kernel, joachims1998text, chapelle1999support, barla2002hausdorff, vishwanathan2010graph}, etc.), statistical principles (\citet{genton2001classes, scholkopf1997prior, osborne2010bayesian}, etc.),  and problems of interest (\citet{gong2024supervised, kokot2025coreset}, etc.). In regression settings, a classical approach is local feature learning, in which kernels are augmented by differential information at points of interest (\citet{schmid1997local, lowe1999object, wallraven2003recognition}). Earlier nonparametric methods developed a similar framework, with datasets being recursively partitioned to improve the quality of local fits \citep{heise1971multivariate, breiman1976general, friedman2006tree}.

A modern incarnation of kernel engineering is {\em multi-index learning}, particularly in the case of neural networks (\citet{mousavi2022neural, boix2023transformers, damian2023smoothing}, etc.). This literature aims to show that the desirable properties of kernel engineering, such as data adaptivity and dimension reduction, are captured implicitly by certain machine learning models. Much work has been done in the {\em single-index} case, where the goal is to regress labels $Y$ on features $X$ when $f$ depends on $x$ solely through its evaluation in a fixed direction $v,$ via the quantity $v^T x$. It has been shown that two-layer neural networks not only learn this dependence, but do so efficiently, leading to rapid increases in prediction quality (\citet{bietti2022learning, abbe2024mergedstaircasepropertynecessarynearly, lee2024neural}, etc.). Recent results indicate that these same benefits carry over to the multi-index setting \citep{damian2023smoothing, arnaboldi2024repetita}.

The SNMH feature structure is typical in manifold learning \citep{aamari2018nonasymptoticratesmanifoldtangent, GenovesePVW12, kokotnoisy}. It also appears under a different name (the ``nonlinear single-variable model'') in the recent work \cite{wu2024conditional}, where the analysis is specialized to a 1-dimensional intrinsic manifold. 
They relate this problem to the more general settings of compositional learning (\citet{Juditsky2009, shen2021deep}, etc.) and non-linear dimension reduction (\citet{yeh2008nonlinear, lee2013general}, etc.). In this literature, many recent methods propose similarly motivated optimization schemes that target related Poincar\'{e} inequalities \citep{bigoni2022nonlinear, verdiere2025diffeomorphism, nouy2025surrogate}.

    Our method bears a particularly strong resemblance to ``kernel steering'' as proposed in the image processing literature \citep{takeda2007kernel,takeda2008deblurring}. 
    The EGOP  has been motivated in many statistical applications \citep{samarov1993exploring, hristache2001structure, xia2002adaptive, trivedi2014consistent, yuan2023efficient}. A precursor in a supervised {\em noiseless} manifold setting is \citet{aswani2011regression}; in this paper, intrinsic rates are achieved with an {\em isotropic kernel} and sparse regression, where the regression coefficients are penalized differently in the normal and tangent space to $\M$. Unlike our work and much of the literature, the method is extended to the regime where the number of regressors tends to infinity. The methods are not adaptive, an outer procedure must estimate the local dimension $d$. 
    
    Besides kernel steering, EGOP metrizations have appeared in many different literatures, with common goals being to accelerate learning tasks and learn efficient dimension reductions \citep{trivedi2014consistent, wu2010learning, mukherjee2010learning}.

\section{The Local EGOP Algorithm}\label{sec:algo}
In this section, we introduce the Local EGOP Learning algorithm, detailed in Algorithm \ref{alg:LEGOP}.

At iteration $t$, the algorithm uses the Gaussian kernel with Mahalanobis metric $M_t$ (step 3) to estimate the local density $\mu_{t}$ about $x^*$. Using this weighting, $\hat{f}$ and its gradients are computed  (steps 5--7), which in turn are used to calculate the AGOP, $L_t$. This is then incorporated into the new metric $M_{t+1} = [\beta L_t + (1-\beta) L_{t-1}]$, $\beta\in (0,1)$.
\begin{algorithm}[tb]
   \caption{Local EGOP Learning}
   \label{alg:LEGOP}
\begin{algorithmic}[1] % [1] adds line numbers
   \REQUIRE Data $\{(X_j,Y_j)\}_{i=1}^n$, target $x^*$, subsample size $m$, iterations $T$, bandwidths $\{h_t\}$, $\beta>0$, scale $\alpha>0$.
   \ENSURE Best estimate $\hat{f}(x^*)$.
   \HYPER $M_0 \gets I/\alpha$.
   
   \vspace{0.1cm} % Visual separation
   \STATE $\text{bestMSE}\gets \infty$, \quad $\hat{f}_{\text{best}}(x)\gets \varnothing$

   \vspace{0.1cm}
   \FOR{$t=0$ {\bfseries to} $T-1$}
       \AlgComment{Compute weights at $x^*$}
       \STATE \(w_i \propto \exp\!\bigl(-(x_i-x^*)^\top M_t (x_i-x^*)\bigr)\)
       $,\ i=1,\dots,n$
       \STATE Normalize $w$, subsample $S\subset\{1,\dots,n\}, |S|=m$ with probs $w$
       
       \AlgComment{Leave-One-Out Local regressions at each $i\in S$}
       \FOR{$i\in S$}
           \STATE Fit weighted local linear regression centered at $x_i$ (weights from metric $M_t$), excluding $(x_i, y_i)$
           \STATE Obtain $\nabla \hat f[i]$ and $\hat{f}[i]$
       \ENDFOR
       
       \AlgComment{Aggregate statistics}
       \STATE $L_t \gets \sum_{i\in S} w_i\, \nabla \hat f[i]\nabla \hat f[i]^\top$
       \STATE $\widehat{\mathrm{MSE}}_t \gets \sum_{i\in S} w_i\,(y_i-\hat{f}[i])^2$

       \vspace{0.1cm}
       \IF{$\widehat{\mathrm{MSE}}_t<\text{bestMSE}$}
           \STATE Estimate $\hat{f}(x^*)$ via local regression centered at $x^*$
           \STATE $\text{bestMSE}\gets\widehat{\mathrm{MSE}}_t$, \quad $\hat{f}_{\text{best}}(x^*)\gets \hat{f}(x^*)$
       \ENDIF

       \AlgComment{Metric update: set $\beta=1$ if $i=0$, $\beta<1$ otherwise}
       \STATE $M_{t+1}\gets\bigl(\beta L_t+(1-\beta)L_{t-1}\bigr)$
       \STATE $M_{t+1}\gets M_{t+1}/[t_{t+1}\operatorname{tr}(M_{t+1})]$
   \ENDFOR

   \vspace{0.1cm}
   \STATE \textbf{return} $\hat{f}_{\text{best}}(x)$
\end{algorithmic}
\end{algorithm}
 In practice, one might choose hyper-parameters $\beta,T$, as well as the schedule $\{h_t\}$ via cross-validation. For computational reasons, we normalize $M_t$ by its trace (step 20) when applying the scaling factor $h_{t+1}$. This has little impact asymptotically, as the trace converges to the leading gradient eigenvalue of the EGOP, making this normalization effectively equivalent to a constant scaling. However, it serves a practical benefit by stabilizing initial iterates, setting the scale of the metric to the predefined rate $1/h_t$. Note that while the optimal number of iterations is dimension-dependent, as argued in Theorems \ref{thm: Generic Rate} and \ref{thm: manifold rate}, our algorithm incorporates a Leave-One-Out style validation loop (step 10) for adaptive stopping.
Further theoretical motivation for the algorithm is presented in Section \ref{sec: intuition}.

The time complexity of Algorithm \ref{alg:LEGOP} is $O(T((m+2)nD^2+(m+1)D^3))$, with memory overhead $O(nD^2+m^2D^2+mD^3)$.

\section{Theoretical analysis}    
\label{sec:theory}
We now study Algorithm \ref{alg:LEGOP} through its continuum $t$ counterparts, described by  $\mu_{t+1} := N(x^*, h_{t+1} M_t^{-1})$ and  $M_t := \beta \mathcal{L}(\mu_t) + (1-\beta) \mathcal{L}(\mu_{t-1})$. 
Setting $\Sigma_t := \Sigma(\mu_t) := \int (x-x^*)(x-x^*)^T\, d\mu_t$, this reduces to the matrix recurrence \begin{align}\label{eq: oracle egop}
\Sigma_{t+1} = h_{t+1} [\beta \mathcal{L}(\mu_t) + (1-\beta)\mathcal{L}(\mu_{t-1})]^{-1}.
\end{align}
In Section \ref{sec:egop-biasvar} we propose a framework for the analysis of the resulting estimator.
We introduce the EGOP form $W(\mu_t) := \operatorname{tr}(\mathcal{L}(\mu_t) \Sigma(\mu_t))$, relating it to the bias of the $M_t$ metrized kernel smoother via a Poincar\'{e} inequality. Building on this, in Section \ref{sec: intuition}, we demonstrate how this recursive kernel adaptation can be naturally framed as a greedy MSE reduction strategy. In Section \ref{sec: guarantees}, our analysis yields an intrinsic learning rate for supervised noisy manifold data, and a significant efficiency gain otherwise.
\subsection{Assumptions and Notation}
    We refer to the population distribution of our feature vectors $X \in \mathbb{R}^D$ as $P$, and its empirical distribution as $P_n$. We assume $P$ has bounded $C^1$ density with respect to the Lebesgue measure on $\mathbb{R}^D$. When invoking the noisy manifold hypothesis, we denote by $\mathcal{M}$ the smooth base manifold of dimension $d$, by $\pi$ the nearest point projection onto the manifold, and by $\mathcal{T}_x$ and $\mathcal{N}_x$ the tangent and normal spaces $U_{\pi(x)}$ and $V_{x}$, respectively, at point $p:=\pi(x)\in \mathcal{M}$. For $\pi$ to be well-defined, we further assume that $X$ lies within the reach of $\mathcal{M}$ \cite{federer1959curvature} almost surely. The labels $Y$ are assumed to be of the form $Y_i = f(X_i) + \varepsilon_i$ for $\varepsilon_i$ iid, mean zero, and finite 4th moment. We assume that $f \in C^4$ and is bounded with uniformly bounded derivatives. We reserve $\mu$ to refer to a localization or target distribution of data appropriately concentrated about the point of interest $x^*$. We denote by $\Sigma$ the covariance of $\mu$, and we enforce $\|\Sigma\| \leq \zeta < \infty$. We denote by $\langle \cdot , \cdot \rangle_F$ the Frobenius inner product.
    We assume $g = \nabla f(x^*) \neq 0$.
    We denote by $k$ a second-order kernel that satisfies typical assumptions in the nonparamtric statistics literature, as in, for example, \cite{Tsybakov2009}[Chapter 1].

    \subsection{Bias Control via EGOP}
    \label{sec:egop-biasvar}
    In this section, we develop a theoretical framework for adaptive kernel learning. Define the empirical kernel estimator for predicting the value of $f$ at $x^*$ by
    \begin{align*}
    \hat{P}_{M_t} (f) &:= \frac{1}{\hat{C}_{M_t}}\frac{1}{n}\sum_{i=1}^n k(\|M_t^{1/2}[X_i-x^*]\|/\sqrt{2h_{t+1}}) Y_i,\\
    \hat{C}_{M_t} &:= \frac{1}{n}\sum_{i=1}^n k(\|M_t^{1/2}[X_i - x^*]\|/\sqrt{2h_{t+1}}).
    \end{align*}
    Its continuous counterpart is 
    \begin{align*}\label{eq:P_M-cont}
    P_{M_t} (f) &:= \frac{1}{C_{M_t}}\int k(\|M_t^{1/2}[y-x^*]\|/\sqrt{2h_{t+1}}) f(y)\, dP(y),\\
    C_{M_t} &:= \int k(\|M_t^{1/2}[y-x^*]\|/\sqrt{2h_{t+1}})\, dP(y).
    \end{align*}
    The bandwidth $h$ is suppressed in our notation as it will be pre-set at each iteration of our algorithm.
    To assess the quality of this predictor, we can measure its bias $\int (P_{M_t}(f) - f)^2\, d\mu$ on a target distribution $\mu$.
    When $M$ and $\mu$ are chosen in tandem, the bias can be related to a Dirichlet form in terms of the EGOP. To make this explicit, we introduce the EGOP form
    \begin{equation}
      W(\mu) := \int \nabla f^T \Sigma(\mu) \nabla f\, d\mu
      =\operatorname{tr}(\mathcal{L}(\mu) \Sigma(\mu)).
    \end{equation}
    \begin{lemma}[Poincar\'{e} Inequality]\label{lem: Poincare}
    Let $M_t\succeq 0$, and set $\mu_{t+1} = N(0, h_{t+1} M_t^{-1})$. Then,
    \[
    \int (P_{M_t}(f) - f)^2\, d\mu_{t+1} = O(W(\mu_{t+1})).
    \]
    \end{lemma}
    In the anisotropic metric $M_t$, the variation of $f$ is given by $\mathcal{L}(\mu_{t+1})$, and the variance of $x$ is given by $\Sigma(\mu_{t+1})$. 
    If these matrices are close to low-rank and perpendicular, then the bias of the kernel estimator is controlled by $W(\mu_{t+1}) = \langle\mathcal{L}(\mu_{t+1}), \Sigma(\mu_{t+1})\rangle_F$, and will be nearly $0$. 
    Of critical importance are the first and second moments of the target distribution $\mu$, motivating  our Gaussian convention, although Gaussianity is not essential to achieve the desired bound. We note that this general structure, of metrizing by the inverse covariance, appears in disparate literatures, from classical local feature estimators \citep{schmid1997local}, to VAEs \citep{chadebec2022geometricperspectivevariationalautoencoders}, and Langevin preconditioning \citep{cui2024optimal}.

    We now seek to control the variance of the finite sample estimator $\hat{P}_M$ of $f$ at $x^*$.
\begin{lemma}[Variance]\label{lem: kernel variance}
    Let $M_t\succeq 0$, $\Sigma_{t+1} = h_{t+1} M_t^{-1}$. Then
    $\mathbb{E}\left[(\hat{P}_{M_t}(f) - P_{M_t}(f))^2 \right] = O\left(1/(n \sqrt{\det \Sigma_{t+1}})\right).$
    \end{lemma}
    Combining these bounds yields our fundamental error rate.
    \begin{theorem}[Local MSE]\label{thm: Generic Rate}
         Let $M_t\succeq 0$, and set $\Sigma_{t+1} = h_{t+1} M_t^{-1}$, $\mu_{t+1} = N(x^*, \Sigma_{t+1})$. Then,
         {\small
         \[
         \mathbb{E} \left[ \int (\hat{P}_{M_t}(f) - f)^2\, d\mu_{t+1} \right] = O\left(\frac{1}{n \sqrt{\det \Sigma_{t+1}}} + W(\mu_{t+1})\right).
         \]
         }
         In particular, if $W(\mu_t) = O(h_t)$ and $\det \Sigma_t = O(h_t^q)$ then the asymptotic learning rate is
         \[
          \mathbb{E} \left[ \int (\hat{P}_{M_t}(f) - f)^2\, d\mu_{t+1} \right] = O \left(n^{-\frac{1}{1+q/2}} \right).
         \]
    \end{theorem}
% Thus, if we design a sequence $\{h_t\}_{t \in \mathbb{N}}$, with $h_t\to 0$,
% such that $M_{t}$ and $\mu_{t}$ are updated recursively based on their
% values up to $t-1$ and $\frac{1}{n \sqrt{\det \Sigma_{h_{t_n}}}} + W(\mu_{h_{t_n}}) \to 0$, the kernel estimator
% $\hat{P}_{M_{h_{t_n}}}$ is consistent and its rate is controlled by the
% convergence rate of $W$ and $M$.

    \subsection{Recursive Kernel Learning as Variance Minimization}
\label{sec: intuition}
We will now show how recursive kernel learning, combined with our basic kernel bounds, constitutes a greedy variance reduction strategy.
Intuitively, the goal is to induce anisotropy in the kernel, stretching or steering it to include additional low-bias points that would otherwise be discarded with an isotropic metric. The inclusion of these additional high-quality points is exactly a reduction of variance, as including more data results in tighter concentration to the mean for a local averaging estimator. 

We now make precise how one can arrive at this by leveraging Theorem
\ref{thm: Generic Rate}. Ideally, as spelled out in this result, we
would like to select $\mu_t$ such that $W(\mu_t)\to 0$, and that $\det \Sigma(\mu_t)$ is as large as
possible. In other words, if we specify the rate $W(\mu_t) =
\operatorname{tr} [\mathcal{L}(\mu_t) \Sigma(\mu_t)] = h_t \to 0$, it
would be optimal to select the largest possible region that achieves
this bias, meaning we would like to select
    \[
    \mu_t = \operatorname{argmax}_{\operatorname{tr} [\mathcal{L}(\mu_t) \Sigma(\mu_t)] \leq h_t} \log \det \Sigma(\mu_t).
    \]
    
    This however cannot be solved, as we lack precise knowledge of the target function $f$. Moreover, $\operatorname{tr} [\mathcal{L}(\mu_t) \Sigma(\mu_t)]\leq h_t$ presents a challenging feasible set, as it is posed over a space of probability distributions, thus making it non-Euclidean. Our Gaussian ansatz reduces the problem to optimizing a finite parameterization, though the optimization is non-convex. 
    
    This leads us to the following discretized relaxation of the problem. Set $\mu_0 = N(x^*,\alpha I)$, an isotropic initialization. If we compute the initial $\hat{\mathcal{L}}(\mu_0)$, we may hope that, as long as $\mu_1$ represents a sufficiently comparable distribution, $\hat{\mathcal{L}}(\mu_0) \approx \hat{\mathcal{L}}(\mu_1)$. Thus, we can compute
    \[
    \Sigma_1 = \operatorname{argmax}_{\operatorname{tr}[\hat{\mathcal{L}}(\mu_0) \Sigma]\leq t_1} \log \det \Sigma = h_1 \hat{\mathcal{L}}(\mu_0)^{-1}/D.
    \]
    This yields $\mu_1 = N(x^*, \Sigma_1)$, and the corresponding kernel estimator is metrized by $\Sigma_1^{-1} \propto \hat{\mathcal{L}}(\mu_0)$. Repeating this, we have a recursive kernel learning algorithm, as we summarize below.

    \begin{proposition}[Alternating Minimization]\label{prop: basic alg}
    Let $h_t\to 0$, fix $\Sigma_0$, $\mu_t = N(x^*, \Sigma_t)$, $\hat{\mathcal{L}}_t:= \hat{\mathcal{L}}(\mu_t)$, and construct $\hat f=\hat{P}_{\Sigma_t^{-1/2}}(f)$. Let $\Sigma_t$ satisfy the recursive kernel recurrence $\Sigma_{t+1} = h_{t+1} \hat{\mathcal{L}}_t^{-1}/D$. Then, equivalently, optimizing over $\Sigma \succeq 0$, 
    \begin{align}
    \Sigma_{t+1} &= \operatorname{argmax}_{\operatorname{tr}[\hat{\mathcal{L}}(\mu_t) \Sigma]\leq h_{t+1}} \log \det \Sigma\\
    &= \operatorname{argmin}_{\operatorname{tr}[\hat{\mathcal{L}}(\mu_t) \Sigma] = h_{t+1}} \operatorname{tr}(\hat{\mathcal{L}}(\mu_t) \Sigma) + \frac{1}{n \sqrt{\det \Sigma}}\nonumber.
    \end{align}
    \end{proposition}

We note that in Algorithm \ref{alg:LEGOP}, we use local linear regression to better fit into the EGOP estimation loop, and we develop corresponding theory in Appendix \ref{app: numerical}.

    From this framework we have developed, it is clear that there are many possible schemes to optimize the bias-variance trade-off, and we discuss this in more detail in Section \ref{sec:discussion}. We choose to focus on this alternating optimization because it strongly resembles algorithms present in the existing literature, and is simple to implement.

\subsection{Bias-Variance trade-off -- Convergence at intrinsic rate}
\label{sec: guarantees}

    In this section, we demonstrate concrete theoretical guarantees for estimation via Local EGOP Learning. For ease of presentation, in the main text we consider an iterative scheme in which an oracle EGOP is queried for a particular choice of covariance. 
    In Appendix \ref{app: numerical} we develop theory towards EGOP approximation, showing that anisotropic smoothing is compatible with recursive kernel learning under standard nonparametric estimation rates.
    For our experiments, we compute empirical AGOPs fit via the proposed iterative estimation procedure, leading to sharp estimation.  See Appendix \ref{app:proofs} for proofs of all following arguments.

%    \paragraph{Taylor Expansion}
    \subsubsection{Taylor Expansion}
    Up to this point, our prevailing perspective is that the target distributions $\mu$ matter up to their first and second moments, passing the optimization problem from the space of measures to the PSD cone, and facilitating a Gaussian relaxation. We close the loop to EGOP estimation by additionally Taylor expanding $\mathcal{L}(\mu)$ up to second order, reducing the proposed iteration scheme entirely to a recurrence in $\Sigma$.

    \begin{lemma}[Generic Taylor Expansion]\label{lem: Generic Taylor}
        Let $\mu = N(x^*,\Sigma)$, $g = \nabla f(x^*)$, and $H = \nabla^2 f(x^*)$. Then there exist $T:\mathbb{R}^{D\times D} \to \mathbb{R}^D$, $R:\mathbb{R}^{D\times D} \to \mathbb{R}^{D\times D}$, $C_T\geq 0,$ and $ C_R\geq 0$ such that for any $A\succeq 0$, $\|T(A)\|\leq C_T \|A\|,\ \|R(A)\|\leq C_R \|A\|^2$, and
        \[
        \mathcal{L}(\mu) = gg^T + H\Sigma H + gT(\Sigma)^T + T(\Sigma) g^T + R(\Sigma).
        \]
    \end{lemma}
    Define $T_g(\Sigma) := gT(\Sigma)^T + T(\Sigma) g^T$. Our goal is to analyze
    \begin{align*}
        \Sigma_{t+1} &= h_{t+1}(\beta \mathcal{L}(\mu_t) + (1-\beta) \mathcal{L}(\mu_{t-1}))^{-1}\\
        &= h_{t+1}(gg^T+ \beta[H \Sigma_t H + T_g(\Sigma_t) + R(\Sigma_t)], \\
        &\quad\quad + (1-\beta)[H \Sigma_{t-1} H + T_g(\Sigma_{t-1}) + R(\Sigma_{t-1})])^{-1}.
    \end{align*}
    We initialize this recurrence with the isotropic choice $\Sigma_0 = \alpha I$, $\Sigma_1 = t_1 \mathcal{L}(\mu_0)^{-1}$, $\Sigma_2 = t_1 \mathcal{L}(\mu_1)^{-1}$, then apply the momentum $\beta$ following this step.

\subsubsection{Full Rank Hessian}
    We begin our analysis with the generic setting, where $H$ is full-rank. Here, the Taylor expansion of $\mathcal{L}(\mu)$ is dominated by a rank 1 matrix $gg^T$ and a linear term $H\Sigma H$, which combine to yield a full rank matrix. 
    We argue that $g^T\Sigma_t g = \Theta(h_t)$, while for $v\perp g,$ $v^T \Sigma_t v = \Theta(\sqrt{h_t})$, thus achieving a second-order anisotropy.

    \begin{example}[Scalar Recurrence]\label{exp: scalar}
        Fix a schedule for $h_t\to 0$ such that $h_{t+1}/h_t$ is larger than a fixed threshold.
        Consider the one dimensional recurrences
        \begin{align}
        a_{t+1} &=  \frac{h_{t+1}}{\beta a_{t} + (1-\beta) a_{t-1}},\label{eq: homogeneous}\\
        b_{t+1} &=  \frac{h_{t+1}}{c + \beta b_{t} + (1-\beta) b_{t-1}} \label{eq: constant part}.
        \end{align}
         It is easy to see that $b_t = O(h_t)$, as if $b_t$ decays at all, then we have $b_t = h_{t+1}/c + o(h_t)$. 
         Assume that $a_t = \Theta(h_t^\zeta)$. By the recurrence relation, we then have $a_i = \Theta(h_t^{1-\zeta})$, and so these coincide at the second order anisotropy $\zeta = 1/2$. 
         As our matrix recurrence is primarily given by a homogeneous part $(H\Sigma H)$ plus a rank deficient constant part $(gg^T)$, we argue in Appendix \ref{app:proofs} that the $b_t$ correspond well to the $g\times g$ block of $\Sigma_t$, and $a_i$ to $g^\perp \times g^\perp$.
    \end{example}

    \begin{example}[Momentum]\label{exp: momentum}
        We now take a moment to highlight the necessity of the momentum term $\beta\in (0,1)$. Consider, again, the sequence $a_t$ from Example \ref{exp: scalar}, and suppose we set $\beta = 1$. For convenience, let us assume $\sqrt{h_{t+1}} = \sqrt{h_t}r$, an exact geometric schedule. We reparameterize, writing
        \[
         a_{t+1} =  \frac{h_{t+1}}{a_{t} } \Longleftrightarrow a'_{t+1} = \frac{1}{a'_t}
        \]
        for $a_i' := \frac{a_i}{\sqrt{r t_i}}$. Clearly, $a_t'$ has no limit, as it oscillates about $1$. In the corresponding matrix recurrence, the remainder terms, as well as the estimation error, make the $\sqrt{h_t}$ rate become unstable. We demonstrate the resulting localizations with and without momentum in Figure \ref{fig: momentum}.
        \begin{figure}
            \centering
            \includegraphics[width=01\linewidth]{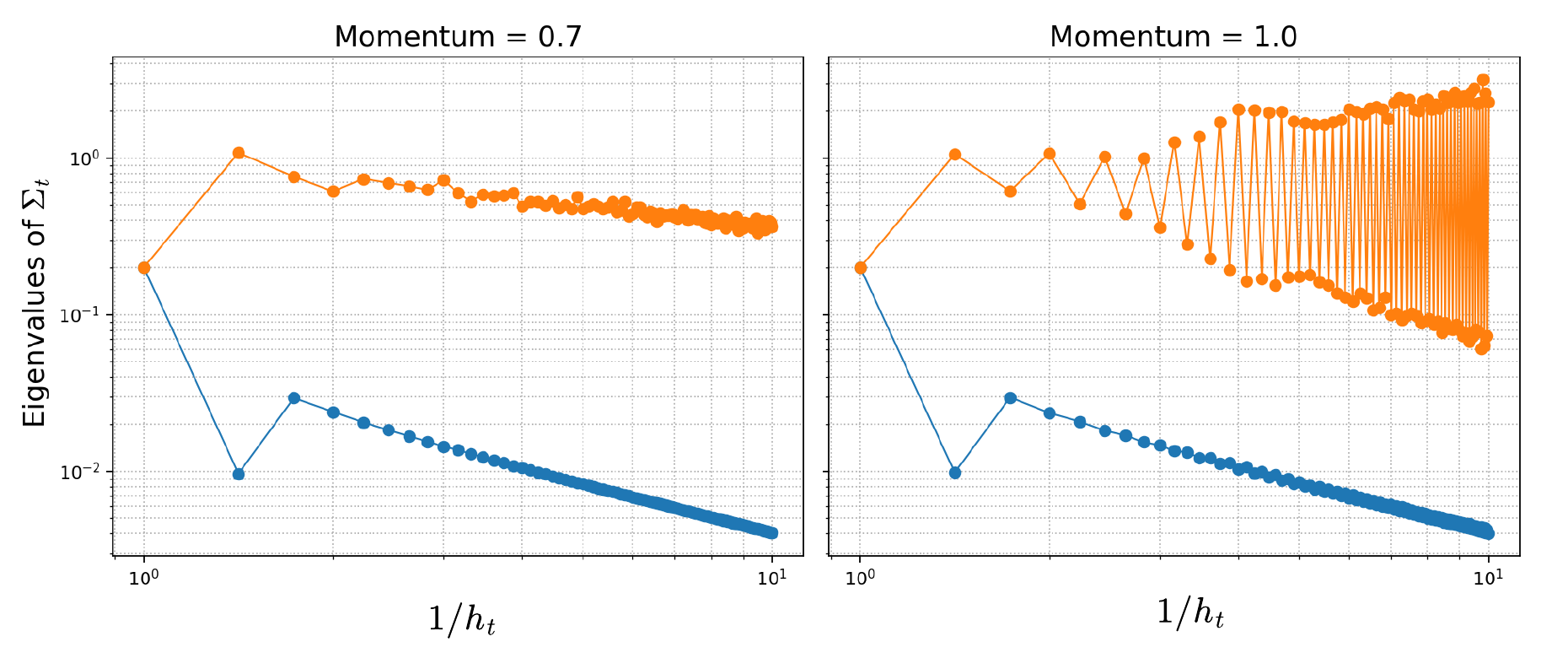}
            \caption{Comparison of the eigenvalue decay of $\Sigma_t$ for Local EGOP Learning with and without momentum. On the left, we set $\beta=0.7$, on the right $\beta = 1$ (no momentum). Without momentum, the second order eigenvalues are prone to wild oscillations.}
            \label{fig: momentum}
        \end{figure}
    \end{example}

     This leads to the following result.

    \begin{theorem}[Second Order Anisotropy]\label{thm: full rank Hessian rate}
    Let $H$ be invertible and $\Sigma_t$ satisfy the recurrence in equation \ref{eq: oracle egop}, $\sqrt{h_{t}} = \alpha r^t$, $0<r<1,\ \beta > 0$. Then, for $\alpha$ sufficiently small, $g^T \Sigma_t g = \Theta(h_t)$, and for $v\perp g$, $v^T \Sigma_t v = \Theta(\sqrt{h_t})$. It follows that, selecting $t_n = \frac{4 \log n}{(D+5) \log(1/r)}$,
    \[
         \mathbb{E}\left[ \int (\hat{P}_{M_{t_n}}(f) - f)^2\, d\mu_{t_n} \right] = O\left(n^{-\frac{4}{D+5}}\right).
    \]
    \end{theorem}

    Compared to a vanilla 0-order local polynomial estimator, this method nearly squares the asymptotic error, effectively cutting the effective dimension in half for $D$ large. 
    In our theory, we assume the initial localization has small covariance $\alpha I$. In the resulting application of Lemma \ref{lem: Generic Taylor}, this results in an explicit recurrence up to a higher-order remainder. While a useful theoretical construct, this is not necessary empirically, and we leave it to future research for a complete analysis of the iteration stability.

    \subsubsection{Noisy Manifold}
In the noisy manifold setting, we strengthen Lemma \ref{lem: Generic Taylor}.
    \begin{lemma}[Noisy Manifold Taylor Expansion]\label{lem: Taylor}
        Let $g = \nabla f(x^*)$, $H = \nabla^2 f(x^*)$, $p = \pi(x^*)$, $\mathcal{T}:= \mathcal{T}_p, \mathcal{N}:= \mathcal{N}_p$, and $T,R$ be as in Lemma \ref{lem: Generic Taylor}. There exist non-negative constants $C_{\mathcal{N}}$ and $ C_{\mathcal{T}\mathcal{N}}$ such that, for all $A\succeq 0$,
        \begin{align}
        \|\pi_{\mathcal{N}} R(A) \pi_{\mathcal{N}}\| &\leq C_{\mathcal{N}}\|\pi_{\mathcal{T}} A \pi_{\mathcal{T}}\|^2,\\
        \|\pi_{\mathcal{T}} R(A) \pi_{\mathcal{N}}\| &\leq C_{\mathcal{T}\mathcal{N}} \|\pi_{\mathcal{T}} A \pi_{\mathcal{T}}\|.
        \end{align}
        Additionally, $\operatorname{rank}(H) \leq 2d.$
    \end{lemma}
    We will leverage this weak orthogonal dependence to show that $\mu_t$ elongates along the space $\mathcal{N}_{x^*}$ orthogonal to the manifold. However, while the function is constant on this fiber, its gradient is not. Indeed, we are only guaranteed that the gradient remains tangent. Due to the effects of curvature, $\nabla f(p+v)$ may shrink or grow along different coordinate directions.
    We characterize this precisely as follows.

    \begin{lemma}[Gradient Geometry]\label{lem: grad form}
    Let $p \in \mathcal{M}$, $\eta \in \mathcal{N}_p$, and $\|\eta\|<\tau$, where $\tau$ denotes the reach of $\mathcal{M}$. For $S_p$ the shape operator,
        \[
        \nabla f(p+\eta) = (I - S_{p}(\eta))^{-1} \nabla f(p).
        \]
        For any such $\eta$, there exists a subspace $\operatorname{Null} = \operatorname{Null}_{p+\eta}$ of dimension at least $D - 2d$ such that, for any $v,z \in \operatorname{Null}$, $\nabla f(p + \eta + v) = \nabla f(p + \eta)$, $\nabla^2 f(p + \eta + v) z = 0$. 
    \end{lemma}

    These geometric implications allow us to further adapt Lemma \ref{lem: Taylor} to our geometry.

    \begin{corollary}[Shifted Taylor]\label{cor: shifted taylor}
        Let $x^* = p + \eta$, $v \in \operatorname{Null}_{p+\eta}$. Define $\mu_v$ as the measure $\mu = N(x^*,\Sigma)$ conditioned on $\pi_{\operatorname{Null}}[X- x^*] = v$, i.e. the draws where the $\operatorname{Null}$ component is fixed at $v$. Let $\Sigma_v$ be the covariance of this distribution. Then,
        \[
        \mathcal{L}(\mu_v) = gg^T + H_v \Sigma_v H_v + g T_v(\Sigma_v)^T + T_v(\Sigma_v) g^T + R_v(\Sigma_v)
        \]
        where $\operatorname{Null}\subseteq \operatorname{nullspace}(H_v)$, $\sup_{v \in \operatorname{Null}}\|R_v(A)\| \leq C_{\operatorname{Null}} \|A\|^2,\ T_v(A) \leq C_{\operatorname{Null}}\|A\|$ for a uniform constant $C_{\operatorname{Null}}$, and $H_v, T_v, R_v$ are continuous in $v$ and Lipschitz in $\|A\|$.
    \end{corollary}

    This allows us to perform an iterated integration, freezing the Taylor expansion along this shifted base point while preserving the fundamental first order behavior.
    To simplify our argument we assume that the covariance cap $\zeta$ is small, allowing a simple reduction of the dynamics to the non-degenerate case. This greatly streamlines the analysis, although we do not believe it is fundamental to the convergence.

    \begin{theorem}[Intrinsic Learning Rate]\label{thm: manifold rate}
        Assume that $D \geq 2d$ and $\operatorname{rank}(H) = 2d$. Let $\Sigma_t$ satisfy the recurrence in equation \ref{eq: oracle egop}, $\sqrt{h_{t}} = \alpha r^t$, $0<r<1,\ \beta > 0$. Then, for $\alpha, \zeta$ sufficiently small, $g^T \Sigma_t g = \Theta(h_t)$. If $v\in \operatorname{Null}$, then $v^T \Sigma_t v = \Theta(1)$, and if $\eta \in \operatorname{Null}^\perp \cap g^\perp$, $\eta^T \Sigma \eta = \Theta(\sqrt{h_t})$. It follows that, selecting $t_n = \frac{4 \log n}{(2d+5) \log(1/r)}$,
    \[
         \mathbb{E} \left[ \int (\hat{P}_{M_{t_n}}(f) - f)^2\, d\mu_{t_n} \right] = O\left(n^{-\frac{4}{2d+5}}\right).
    \]
    \end{theorem}

    Rather than fully denoising the manifold, our analysis reveals that the manifold tangent, as well as the normal vectors that do not fall in $\operatorname{Null}$, are subject to the same second-order anisotropy as observed in the full rank setting.
    These compose precisely those directions in which the gradient varies at first order. The resulting rate is comparable to that of a standard estimator applied to a completely denoised manifold, yielding an effective dimension of $d+1/2$.

    \begin{example}[Multi-Index Learning]
        When the underlying manifold $\mathcal{M}$ is flat, the supervised noisy manifold hypothesis reduces to a multi-index setting. The key distinction is that displacement away from the manifold not only preserves the function value, but also all higher order function derivatives. In particular, the Hessian has rank no more than $d$. 
        When performing expansions such as Corollary \ref{cor: shifted taylor}, the orthogonal has no contribution to the EGOP, yielding a learning rate of $(d+1)/2$.
        
    \end{example}

    \begin{example}[Continuous Single-Index]
When the underlying manifold is 1-dimensional, the tangent space coincides with the span of the gradient, hence $\pi_{g^\perp} H \pi_{g^\perp} = 0$, removing the second order anisotropy. Thus no normal directions contract asymptotically, leading to an effective dimension of 1 for the learning rate. This corroborates the result of \citet{wu2024conditional}.
\end{example}

    \section{Experiments}\label{sec:expe}

    Details for data generation, evaluation, and training of each method can be found in Appendix \ref{app:expe}.

    \begin{figure}[tb]
        \centering
        \includegraphics[width=0.9\linewidth]{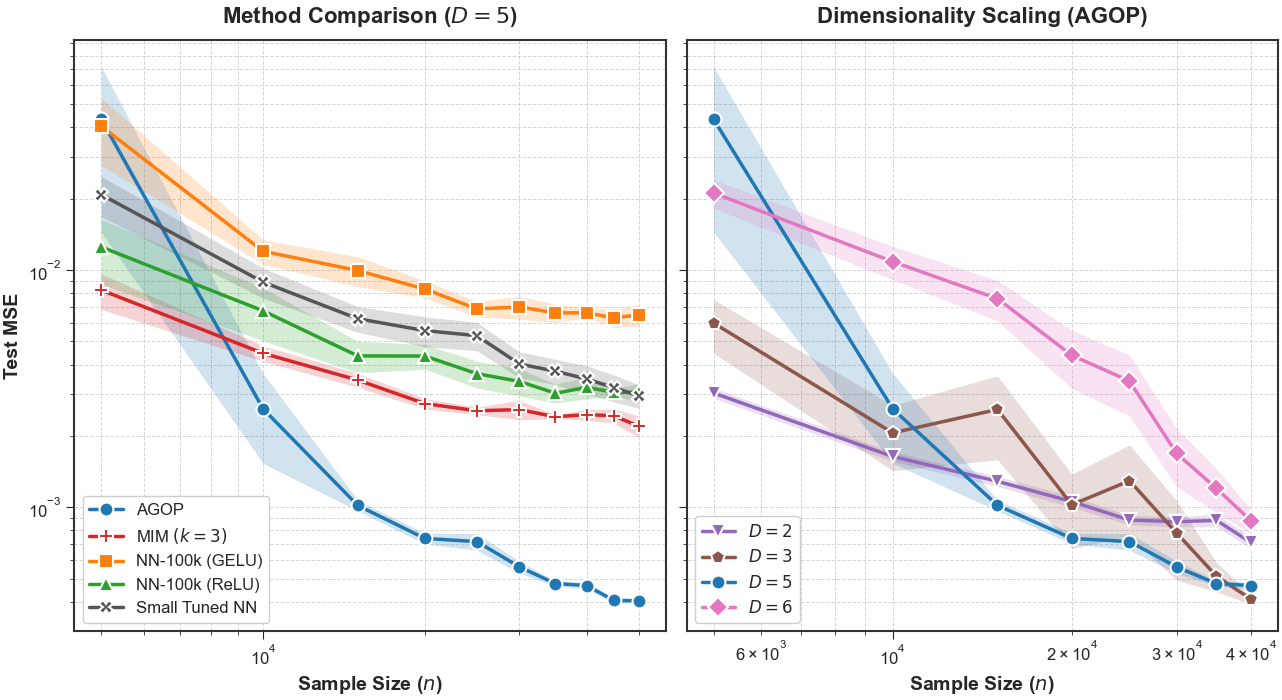}
        \caption{(Left) Comparison of Local EGOP Learning to the performance of two-layer neural network architectures trained on helical data in ambient dimension $D=5$. (Right) Local EGOP Learning trained on helical data in various ambient dimensions $D$.}
        \label{fig: combined}
    \end{figure}
    
    \subsection{Two-layer Neural Network}\label{sec: two-layer}
    
    We assess a variety of two-layer neural network architectures, with details in Appendix \ref{app:expe}.
    The covariate distribution is drawn from helical data, which is continuous single-index, being drawn from a noise contaminated curve, as well as multi-index 2 for $D$ even, and 3 for $D$ odd. Further, by construction it is well-approximated by a low intrinsic dimensional manifold, a typical setting for accelerated learning rates (\citet{kiani2024hardness,liu2021besov}, etc.).

    As demonstrated in Figure \ref{fig: combined}, the two-layer neural networks readily adapt to the multi-index data structure. For $D=5$, the overparameterized ReLU network achieves comparable error to a multi-index architecture pre-specified to learn a 3-dimensional subspace. However, Local EGOP Learning vastly reduces the error for moderate to large sample sizes. This suggests that no architecture was able to fully adapt to the intrinsic continuous single-index structure.

    \subsection{Learning Rate}
    
    We generate helical data in a variety of dimensions, testing the Local EGOP Learning algorithm. As seen in Figure \ref{fig: combined}, the asymptotic learning rate has little dimensional dependence. However, error scaling depends strongly on whether $D$ is even or odd, and this is because the underlying manifold changes in these settings, although its intrinsic dimension does not.

    \subsection{Feature Learning}
    
    In this section, we compare the local feature learning capabilities of a deep transformer-based neural network \citep{gorishniy2023embeddingsnumericalfeaturestabular} to Local EGOP Learning. We consider data generated from a 1-sphere under the supervised noisy manifold hypothesis. On this simple dataset, both algorithms result in nearly complete denoising of the data, as shown in Figures \ref{fig:agop_descent} and \ref{fig:deep_feet}. In Appendix \ref{app:expe}, we provide additional unsupervised embeddings for comparison.

    \subsection{Predicting the backbone angles in Molecular Dynamics (MD)  data}
    \label{sec:md}
    
    Our final example comes from the analysis of molecular geometries.
    The raw data consist of $X,Y,Z$ coordinates for each of the $N_a$ atoms of a molecule, which, due to interatomic interactions, lie near a low-dimensional manifold \citep{dasMSKClementi:06}.
    While the governing equations of the simulated dynamics are unknown, for small organic molecules, certain backbone angles \citep{dasMSKClementi:06} have been observed to vary along the aforementioned low-dimensional manifold. Specifically, for the malonaldehyde molecule, the two backbone angles denoted $\tau_{1,2}$ are shown in Figure \ref{fig:malon}. We used a subsample of molecular configurations of size $n=10^4$ from the MD simulation data of \citet{chmielaTkaSauceSchuPMull:force-fields17} as input data.
    The configuration data, pre-processed as in \citet{koelleZhangMchenyuchia:jmlr22}, consist of $D=50$ dimensional vectors and lie near a 2-dimensional surface with a torus topology (see Figure \ref{fig:malon}). On a hold-out set of 500 test points, Local EGOP Learning yields an MSE of 0.00043, compared to 0.012 for Gaussian kernel Nadaraya-Watson with cross-validated bandwidth selection.

\section{Discussion and Future Work}\label{sec:discussion}

The problem of learning an EGOP \citep{trivedi2014consistent, radhakrishnan2025linear}, its local counterpart \citep{takeda2007kernel}, or gradients \citep{mukherjee2010learning, wu2010learning}, has been a topic of frequent theoretical interest, with applications to image processing \citep{takeda2007kernel, takeda2008deblurring}, machine learning \citep{radhakrishnan2022mechanism, beaglehole2023mechanism}, and dimension reduction \citep{yuan2023efficient}.
Our research establishes a new functional framework to understand these methods, directly relating the EGOP to the MSE learning objective, and recursive kernel learning to a greedy variance reduction strategy.
The resulting algorithm yields not only new efficiency guarantees for point estimation, but also tractable access to a local EGOP metrization.

In this paper, we introduce continuous index learning and the closely related supervised noisy manifold hypothesis. We demonstrate that Local EGOP Learning is able to achieve an intrinsic dimensional learning rate for such data. For generic manifolds, however, it does not completely denoise the covariate distribution.
Our theory shows that a small subspace of the orthogonal will contract at the same rate as the tangent, despite the invariance of $f$. 
This counterintuitive result reflects the fact that our EGOP optimization scheme is derived via an upper-bound on the bias, and is thus not always sharp.
The refinement of this methodology to close this gap is thus of both practical and theoretical importance.

The Local EGOP Learning algorithm is \textit{adaptive}, in that it outperforms standard kernel regression algorithms when the noisy supervised manifold hypothesis fails to hold (Theorem \ref{thm: full rank Hessian rate}), and yields an intrinsic dimensional rate when $f$ is structured. 
In neither case is knowledge of the regularity of $f$ or the latent manifold required to enhance estimation quality. Importantly, this means that users may use the algorithm without a priori knowledge of problem dimensionality, noise level, or whether the supervised manifold hypothesis holds.

 Algorithm \ref{alg:LEGOP} was presented for clarity, not efficiency; there are many possible computational improvements, including partitioning schemes as in \citet{beaglehole2025xrfm}.
 Local EGOP Learning is based firmly on a MSE formulation of regression, restricting its utility to $\ell_2$ loss. This method requires suitable modification to be applied directly to problems such as classification or image generation. This leaves open how an appropriate adaptation should be formulated in these settings. 

% As discussed in Appendix \ref{app:discussion}, in the context of Local EGOP Learning, the problem of EGOP estimation is left open. 
% Algorithm \ref{alg:LEGOP} is slow, requiring many local linear regressions at each iteration. We presented this simple algorithm for clarity, and many practical improvements can be made to improve its performance, particularly when multiple data points are queried for prediction. Furthermore, simple modifications to the algorithm would yield gains in statistical efficiency. Higher-order local polynomial regression can be performed for function estimation. This would improve theoretical guarantees by accounting for more general smoothness patterns in the function of interest, but it may also inform more complex localization procedures. Gaussian localizations are quadratic in nature, and so a more generic formulation may yield increased adaptivity. 
    
    Our methodology is primarily framed through a second order, Gaussian relaxation. One promising direction is to drop this restriction, and view $W(\mu_t)$ and $\Sigma(\mu_t)$ as functionals, and to then optimize the MSE using a Wasserstein gradient flow \citep{ambrosio2006gradient} or other methods that operate at the distribution level. In particular, this would allow for non-ellipsoidal localizations, potentially allowing for improved learning rates in settings where the level-set is curved. A fundamental hurdle for any of these approaches is that, generally, the EGOP form can only upper bound localization-induced bias, and the resulting bound is not always sharp.

\subsubsection*{Acknowledgements}
The authors would like to thank Misha Belkin for the initial suggestion of comparing Local EGOP Learning to the performance of two-layer neural networks, and Brennan Dury for suggestions and experiments related to earlier conceptions of this research.

% \input{egop-acknowledgements.tex}
% \section*{Impact Statement}
% This paper presents work whose goal is to advance the field of Machine
% Learning. There are many potential societal consequences of our work, none
% which we feel must be specifically highlighted here.
        
\bibliography{refs}
\bibliographystyle{icml2026}

%%%%%%%%%%%%%%%%%%%%%%%%%%%%%%%%%%%%%%%%%%%%%%%%%%%%%%%%%%%%%%%%%%%%%%%%%%%%%%%
%%%%%%%%%%%%%%%%%%%%%%%%%%%%%%%%%%%%%%%%%%%%%%%%%%%%%%%%%%%%%%%%%%%%%%%%%%%%%%%
% APPENDIX
%%%%%%%%%%%%%%%%%%%%%%%%%%%%%%%%%%%%%%%%%%%%%%%%%%%%%%%%%%%%%%%%%%%%%%%%%%%%%%%
%%%%%%%%%%%%%%%%%%%%%%%%%%%%%%%%%%%%%%%%%%%%%%%%%%%%%%%%%%%%%%%%%%%%%%%%%%%%%%%
\newpage

\onecolumn 
\appendix

\begin{center}
    {\Large \textbf{Local EGOP for Continuous Index Learning: Supplementary Material}}
\end{center}
\bigskip

\section{Experiment Details}\label{app:expe}

The simulations were run on a machine of Intel\textsuperscript{\textregistered} Xeon\textsuperscript{\textregistered} processors with 48 CPU cores, and 50GB of RAM.

\begin{figure}[ht]
    \centering
    \begin{subfigure}[b]{0.55\linewidth}
        \centering
        \includegraphics[width=\linewidth]{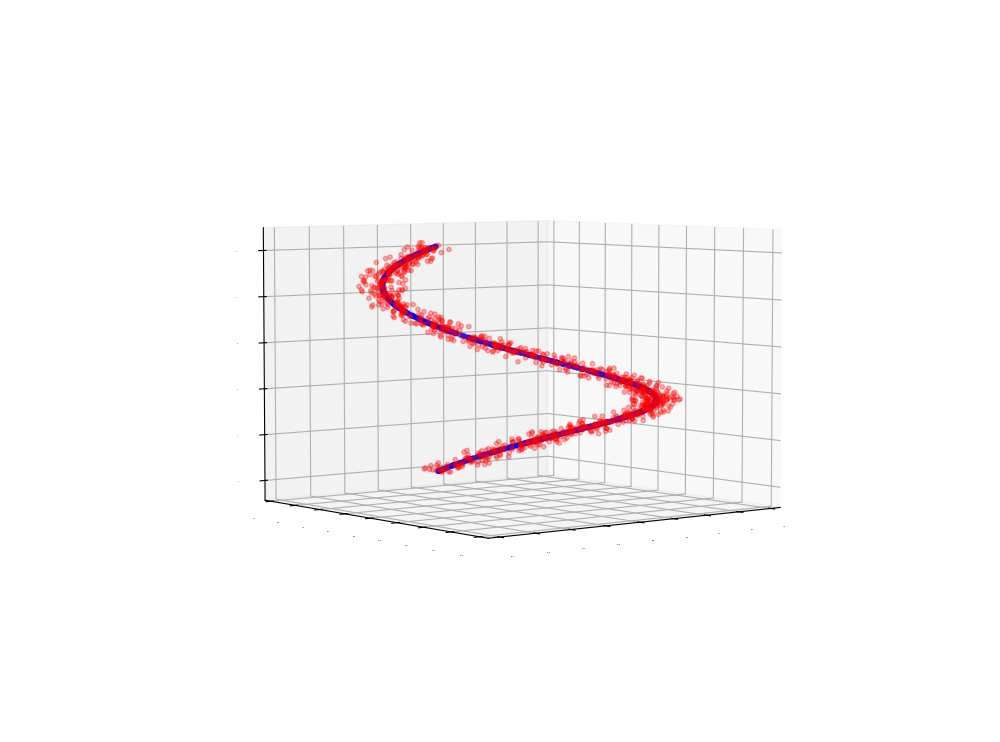}
        \caption{3-dimensional Helical data.}
        \label{fig:helix}
    \end{subfigure}
    \hfill
    \begin{subfigure}[b]{0.4\linewidth}
        \centering
        \includegraphics[width=\linewidth]{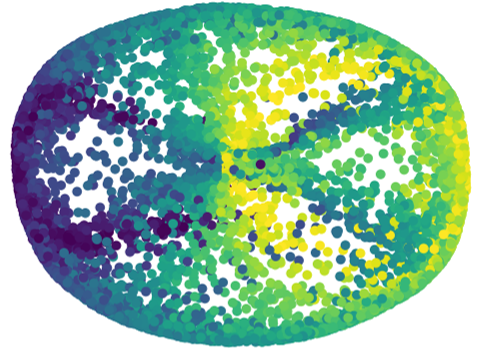}
        \caption{Coordinates of Malonaldehyde embedded into 2 dimensions via Diffusion Maps \citep{coifman2006diffusion} and colored by the backbone angles $\tau_1$.}
        \label{fig:malon}
    \end{subfigure}
    \caption{Visualizations of the experimental datasets.}
    \label{fig:helix_malon}
\end{figure}

\subsection{General Set-up}

The following simulations to empirically justify our theoretical results. First, expanding on the noisy manifold setting, we show how the regression error rates of Local EGOP Learning remain invariant to the injection of high-dimensional noise. We demonstrate that two-layer neural networks are not able to efficiently learn in the noisy manifold setting, with a sharp decrease in performance compared to Local EGOP Learning. We next compare the performance of our method to that of a deep neural network on toy data. In particular, we show how the feature embeddings produced by transformers trained on a simple example are qualitatively similar to the localizations generated by our procedure. Finally, we apply this procedure to estimate backbone angles in Molecular Dynamics (MD) data, where we leverage the noisy manifold structured data to improve prediction quality.

In our simulations, we consider helical data, parameterized by a curve $\theta(t) = (\sin(t + w_1), \cos(t+w_1), \sin(t+w_2), \cos(t+w_2),\dots, g(t))$, where $g(t)=t$ is a linear term included if $D$ is odd dimensional, and the $w_i$ are constant offsets taken as a mesh from $0$ to $2\pi$. We rescale this data by a constant $\tau$, then contaminate it with uniform orthogonal noise of radius $r$. In our simulations, we set $\tau = 0.8,\, r = 0.5$, and sample $t$ from 0 to $2\pi$. See Figure \ref{fig:helix} for a visualization. For the outcomes $y$, we generate a third-degree polynomial with coefficients uniformly sampled from $(-3, 3)$, then evaluate it at the projection point onto $\theta(t)$.

In our synthetic experiments, to generate labeled data we randomly generated a homogeneous cubic polynomial
\[
f(x) = \sum_{|\alpha| \leq 3} c_\alpha x^\alpha,
\]
where $c_\alpha \sim \mathrm{Uniform}[-3, 3]$. The resulting function $f$ is then evaluated at the unperturbed datapoints before the features $X$ were contaminated with noise orthogonal to the underlying manifold $\mathcal{M}$. The coefficients $c_\alpha$ are randomly sampled across each replicate of the experiments. In our plots, confidence bands are displayed with a width of $\pm 1$ standard error. We generically apply the below hyperparameters to the Local EGOP Learning algorithm in each example.
\begin{itemize}
    \item Number of EGOP iterations \( T_{\text{test}} = 150 \),
    \item Monte Carlo sample size \( m = 300 \),
    \item Bandwidth decay $h_t = (1+t)^{-1.2}$,
    \item Initialization scale \( \alpha = 0.2 \),
    \item Exclusion radius \(\varepsilon = 1.6\).
\end{itemize}

\subsection{Two-layer Neural Network}
We train fully connected NNs with both ReLU and GELU activations. All first layers use Kaiming initialization, and output layers are initialized at variance $1/\sqrt{m}$ where $m$ is hidden width. All networks are trained with Adam.

For moderate-scale neural models, we conduct a grid search over hyperparameters. Widths are selected from $\{128, 256, 512, 1024, 2048, 4096, 8192\}$, learning rates are selected from $\{10^{-3}, 3 \times 10^{-3}\}$, weight decay from $\{0, 10^{-5}, 10^{-4}\}$, and batch size from $\{16, 32, 64\}$. The default Adam momentum parameters are used. Early stopping is applied with patience $200$ and improvement tolerance $10^{-4}$ monitored on a validation split. In each experiment, only the optimal test error of these methods is aggregated into the ``Small Tuned NN'' plot.

We train overparameterized two-layer models with width $100{,}000$, Adam with learning rate $0.003$, batch size $64$, and fixed weight decay $10^{-4}$, applying early stopping with patience $300$ and improvement tolerance $10^{-4}$.

We also consider multi-index models $f(x) = g(U x)$, $g$ a two-layer NN, with width 256 and GELU activations, $U$ a rank three projection matrix updated at each iteration.

\subsection{Feature Learning}

We generated $n = 50{,}000$ samples on a spherical cap in $\mathbb{R}^2$, restricted to an angular width of $25^\circ$. Each point $x_i$, we evaluate $f(x_i)$ for $f$ a cubic polynomial with $\mathrm{Uniform}[-1,1]$ coefficients. The features $x_i$ are then scaled radially by a random factor $r_i \sim \mathrm{Uniform}[r, R]$, with $r = 0.6$ and $R = 1.8$. Additive Gaussian noise with standard deviation $0.05$ was applied to the labels. We set $x^* = (1, 0)^T$.

We trained an \texttt{FTTransformer} regression model \citep{gorishniy2021revisiting} on this dataset using the \texttt{rtdl} and \texttt{zero} frameworks. The network was trained for 40 epochs with a batch size of 1024 and learning rate $2 \times 10^{-3}$. After every 50 training steps, we computed weights:
\[
w_i = \frac{\exp(-\|z_i - z^*\|^2 / 8)}{\sum_j \exp(-\|z_j - z^*\|^2 / 8)},
\]
where $z_i$ denotes the learned feature representation of $x_i$.

\subsection{Unsupervised Embedding}

As noted in works such as \citep{coifman2006diffusion, elkaroui2010spectrum, landa2023robust, kokotnoisy}, unsupervised methods such as diffusion maps are robust to noise. The resulting low-frequency eigenfunctions closely match those of the underlying manifold. 
However, as demonstrated in Figure \ref{fig: DM}, they fail to capture this structure exactly, and thus cannot match the metrization produced by Local EGOP Learning and deep learning architectures in Figure \ref{fig:agop_descent}.

\begin{figure}[tb]
    \centering
    \includegraphics[width=0.5\linewidth]{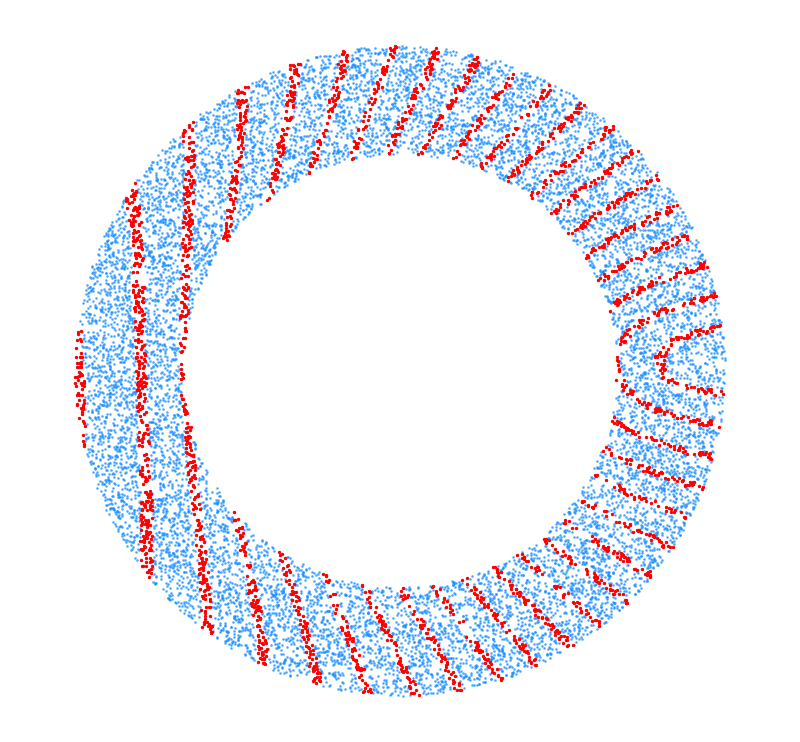}
    \caption{Data sampled from an annulus, with red rings displaying uniform distances in diffusion map embedding distance from the point $(1,0)^T$, for a 2-dimensional embedding.}
    \label{fig: DM}
\end{figure}

\section{Proofs}\label{app:proofs}

In our theory, we focus on the kernel $k = \exp(-\|x\|^2)$. We discuss general kernels in Appendix \ref{app: general kernels}.

\subsection{Bias Control via EGOP}

\subsubsection{Bias Bound}\label{app: bias bound}

To derive the upper bound from Lemma \ref{lem: Poincare}, we argue in stages. We first decompose the kernel smoother into two parts. The first is given by a pure Gaussian convolution and a density dependent covariance relating $f$ and $p$. Both terms can readily be controlled via the EGOP form, leading to our bound. In the following, we absorb the bandwidth into the metric and covariance, suppressing $h$ in the notation.

Define $\mu_{M} : = N(x^*, M^{-1}),$ $\mathbb{E}_{M}[\cdot ] := \mathbb{E}_{\mu_{M}}[\cdot ]$, $\operatorname{Cov}_M$, $\operatorname{Var}_M$ similarly.

\begin{lemma}[Basic Decomposition]
    \[
    P_M f = \mathbb{E}_{M} f + \frac{\operatorname{Cov}_{M}[f, p]}{\mathbb{E}_{M} p} =: G_M f + A_M f
    \]
\end{lemma}

\begin{proof}
    We can simply compute
    \begin{align*}
    P_M f &= \frac{\mathbb{E}_{M}[fp]}{\mathbb{E}_{M}[p]} = \frac{\mathbb{E}_{M}[f]\mathbb{E}_{M}[p]}{\mathbb{E}_{M}[p]} + \frac{\operatorname{Cov}_M[f, p]}{\mathbb{E}_{M}[p]}\\
    & = \mathbb{E}_{M}[f] + \frac{\operatorname{Cov}_M[f, p]}{\mathbb{E}_{M}[p]} =: G_M[f]+ A_M f.
    \end{align*}
\end{proof}

\begin{lemma}[Gaussian Poincar\'{e}]\label{lem: Gaussian Bias}
    \[
    \int (\mathbb{E}_{M}[f] - f)^2\, d\mu_M \leq W(\mu_M)
    \]
\end{lemma}
\begin{proof}
    This is simply the Gaussian Poincar\'{e} inequality, see for instance \cite{bakry2013analysis}[Proposition 4.1.1].
\end{proof}

\begin{lemma}[Poincar\'{e} $\times$ $p$ variance]\label{lem: density bias}
    \[
    (A_M f)^2 \leq W(\mu_M) \frac{\operatorname{Var}_M(p)}{\mathbb{E}_{M}[p]^2}
    \]
\end{lemma}

\begin{proof}
    \begin{align*}
    A_M f^2 &= \frac{\operatorname{Cov}_M[f, p]^2}{\mathbb{E}_{M}[p]^2} \leq \frac{\operatorname{Var}_M(f) \operatorname{Var}_M(p)}{\mathbb{E}_{M}[p]^2} \leq W(\mu_M) \frac{\operatorname{Var}_M(p)}{\mathbb{E}_{M}[p]^2},
    \end{align*}
    applying the Gaussian Poincar\'{e} inequality \cite{bakry2013analysis}[Proposition 4.1.1] for the last line.
\end{proof}

\begin{proof}[Proof of Lemma \ref{lem: Poincare}]
We can apply Lemmas \ref{lem: Gaussian Bias} and \ref{lem: density bias} immediately to yield
\begin{align*}
    \int (P_M - f)^2\, d\mu_M &= \int (\mathbb{E}_{M}[f] - f + A_M f)^2\, d\mu_M = \int (\mathbb{E}_{M}[f] - f)^2\, d\mu_M + A_M f^2 \\
    &\leq \left(1 + \frac{\operatorname{Var}_M(p)}{\mathbb{E}_{M}[p]^2}\right)W(\mu_M).
\end{align*}
Note that the density dependent factor $\frac{\operatorname{Var}_M(p)}{\mathbb{E}_{M}[p]^2}$ is bounded by the continuity of $p$ and the capped covariance assumptions.
\end{proof}

\subsubsection{Variance Bound}

We operate by a typical Nadaraya-Watson style argument. Our first steps are to show that the denominator of our estimator concentrates, allowing us to treat it as a constant up to a negligible residual. Let $\gamma_M \propto \det(M)^{1/2}$ by the normalization factor of $\mu_M$. Define
\[
\hat{C}'_M := \gamma_M \hat{C}_M := \gamma_M \frac{1}{n} \sum_{i=1}^n \exp(-(X_i - x^*)^T M (X_i - x^*) ).
\]

\begin{lemma}[Denominator Concentration]\label{lem: denom concentration}
    \[
    \mathbb{P}(|\hat{C}'_M - \mathbb{E}_{M}(p)| \leq \mathbb{E}_{M}(p)/2) \leq 2\exp\left(-\frac{n \mathbb{E}_{M}(p)}{10 \gamma_M}\right)
    \]
\end{lemma}
\begin{proof}
    We prove this by applying Bernstein's inequality. First, observe $\mathbb{E}[\hat{C}_M] = G_M(p)$. Next, $\exp(-v^T M v) \in (0, 1]$ for any $v$, hence $\hat{C}'_M = \frac{\gamma_M}{n} \sum_{i=1}^n \exp(-(X_i - x)^T M (X_i - x) )$ is the sample mean of iid bounded random variables in $(0, \gamma_M]$. Further, we can compute
    \begin{align*}
    \operatorname{Var}[\gamma_M \exp(-(X_i - x)^T M (X_i - x) )] &\leq \gamma_M \mathbb{E}[\gamma_M \exp(-(X_i - x)^T M (X_i - x) )^2]\\
    &\leq \gamma_M \mathbb{E}[\exp(-(X_i - x)^T M (X_i - x) )] = \gamma_M G_M(p).
    \end{align*}
    Hence,
    \begin{align*}
        \mathbb{P}(|\hat{C}'_M - G_M(p)|&\leq - G_M(p)/2) \\
        &\leq 2 \exp\left(- \frac{n G_M(p)^2}{8 \operatorname{Var}(\gamma_M \exp(-(X_i - x)^T M (X_i - x) )) +\frac{4}{3} \gamma_M G_M(p)^2  }\right)\\
        &\leq 2 \exp\left(- \frac{n G_M(p)^2}{8 \gamma_M G_M(p) +\frac{4}{3} \gamma_M G_M(p)  }\right) \\
        &\leq 2 \exp\left(- \frac{n G_M(p)}{10\gamma_M} \right).
    \end{align*}
\end{proof}

\begin{lemma}[Full Variance Expression]\label{lem: variance full}
    Assume that the noise $\varepsilon_i$ have finite fourth moments. Then, there exists a constant $C>0$ such that
    \begin{align*}
    \mathbb{E}[(\hat{P}_M(f) - P_M(f))^2] &\leq 2 \frac{ \mathbb{E}[\gamma_M^2 \exp(-(X_i - x)^T M (X_i - x))^2\{Y_i - P_M(f)\}^2]}{n G_M(p)^2}\\
    &\quad\quad + 2\frac{\mathbb{E}[\{\gamma_M \exp(-(X_i - x)^T M (X_i - x)) Y_i\}^2]}{nG_M(p)^2} + C \exp\left(- \frac{n G_M(p)}{10\gamma_M} \right)
    \end{align*}
\end{lemma}

\begin{proof}
    Let $E$ be the event $|\hat{C}_M - G_M(p)| \leq G_M(p)/2$. We compute
    \begin{align*}
    \hat{P}_M(f) - P_M(f) &= \frac{\frac{1}{n}\sum_{i=1}^n \gamma_M\exp(-(X_i - x^*)^TM(X_i - x)) Y_i}{\hat{C}'_M} - \frac{G_M(fp)}{G_M(p)}\\
    & = \frac{\frac{1}{n}\sum_{i=1}^n  [\gamma_M\exp(-(X_i -x^*)^TM(X_i -x^*)) \{Y_i - P_M(f)\}]}{G_M(p)} \mathbf{1}_E\\
    &\quad\quad + \frac{G_M(p) - \hat{C}'_M}{G_M(p)\hat{C}'_M} \frac{1}{n}\sum_{i=1}^n [\gamma_M\exp(-(X_i -x^*)^TM(X_i -x^*)) Y_i] \mathbf{1}_E\\
    &\quad \quad + \left\{\frac{\frac{1}{n}\sum_{i=1}^n \gamma_M\exp(-(X_i -x^*)^TM(X_i -x^*)) Y_i}{\hat{C}'_M} - \frac{G_M(fp)}{G_M(p)}\right\} \mathbf{1}_{E^c}.
    \end{align*}
    Hence,
    \begin{align*}
     \mathbb{E}[(\hat{P}_M(f) - P_M(f))^2] &\leq 2 \frac{ \mathbb{E}[\gamma_M^2 \exp(-(X_i -x^*)^T M (X_i -x^*))^2\{Y_i - P_M(f)\}^2]}{n P_M(p)^2}\\
     &\quad\quad + 2\mathbb{E}\left[\left\{ \frac{G_M(p) - \hat{C}'_M}{G_M(p)\hat{C}'_M} \frac{1}{n}\sum_{i=1}^n [\gamma_M\exp(-(X_i -x^*)^TM(X_i -x^*)) Y_i]\right\}^2 \mathbf{1}_E\right]\\
     &\quad\quad + 2\mathbb{E}\left[\left\{\hat{P}_M(f) - P_M(f) \right\}^2 \mathbf{1}_{E^c}\right].
    \end{align*}
    The first term is as desired. Next, by definition of the event, we have
    \[
    \mathbb{E}\left[\left\{ \frac{G_M(p) - \hat{C}'_M}{G_M(p)\hat{C}'_M} \frac{1}{n}\sum_{i=1}^n [\gamma_M \exp(-(X_i -x^*)^TM(X_i -x^*)) Y_i]\right\}^2 \mathbf{1}_E\right] \leq \frac{\mathbb{E}[\left\{\gamma_M \exp(-(X_i -x^*)^TM(X_i -x^*)) Y_i\right\}^2]}{n G_M(p)^2}.
    \]
    For the remaining term, observe that
    \[
    \hat{P}_M(f) = \sum_{i=1}^n w_i Y_i
    \]
    where $w_i$ are a convex combination. Hence,
    \[
    \mathbb{E}\left[\left\{\hat{P}_M(f) - P_M(f) \right\}^2 \mathbf{1}_{E^c}\right] \leq \sqrt{\mathbb{E}\left[\left\{2\hat{P}_M(f)^2 + 2P_M(f)^2 \right\}^2\right]} \mathbb{P}[E^c]\leq C \exp\left(- \frac{n G_M(p)}{10\gamma_M} \right),
    \]
    by Lemma \ref{lem: denom concentration} and the hypothesis on the moments of $\varepsilon_i$ and the boundedness of $f$.
\end{proof}

With this out of the way, we bound the two remaining expectations.

\begin{lemma}[Normalization Constant]\label{lem: normalization}
    There exists a constant $C>0$ such that
    \begin{align*}
    \mathbb{E}[\gamma_M^2 \exp(-(X_i -x^*)^T M (X_i -x^*))^2\{Y_i - P_M(f)\}^2] &\leq C \gamma_M G_{2M}(p),\\
    \mathbb{E}[\{\gamma_M \exp(-(X_i -x^*)^T M (X_i -x^*))Y_i\}^2] &\leq C \gamma_M G_{2M}(p).
    \end{align*}
\end{lemma}
\begin{proof}
    Notice that $\gamma_M/\gamma_{4M} = C$ for some constant. Let $\varepsilon \sim \rho$. Focusing on the first expression,
    \begin{align*}
        &\mathbb{E}[\gamma_M^2 \exp(-(X_i -x^*)^T M (X_i -x^*))^2\{Y_i - P_M(f)\}^2]\\
        &\quad\quad = C \gamma_M  \int  \gamma_{4M}\exp(-(z-x)^T[2M] (z-x)) p(y)(f(y) + \varepsilon - P_M(f))^2\, dz\, d\rho(\varepsilon)\\
        &\quad\quad = C \gamma_M \mathbb{E}_{2M}[p(x+Z) \mathbb{E}_\rho [ (f(x+Z) + \mathcal{E} - P_M(f)(x))^2]] \leq C' \gamma_M G_{2M}(p)(x),
    \end{align*}
    using bounded moments and bounded $f$ assumptions. Similarly,
    \begin{align*}
        \mathbb{E}[\{\gamma_M \exp(-(X_i -x^*)^T M (X_i -x^*))Y_i\}^2] & = C \gamma_M \mathbb{E}_{2M, \rho}[p(Z) [f(Z) + \mathcal{E}]^2 ] = C\gamma_M G_{2M}(p).
    \end{align*}
\end{proof}

Combining these yields the desired variance bound multiplied by density dependent factors.

\begin{proof}[Proof of Lemma \ref{lem: kernel variance}]
    By Lemmas \ref{lem: variance full} and \ref{lem: normalization}, there exists a constant $C> 0$ such that
    \begin{align*}
    \mathbb{E}[(\hat{P}_M(f) - P_M(f))^2] &\leq C \frac{\gamma_M}{n}\left(\frac{G_{2M}(p)}{G_M(p)^2}\right) + C \exp\left(-\frac{n}{\gamma_M}\left(\frac{G_M(p)}{10}\right)\right)\\
    &= O\left(\frac{\sqrt{\det M}}{n}\right).
    \end{align*}
\end{proof}

\subsubsection{MSE}

We compile the results from the previous section to provide guarantees for the MISE against Gaussians.

\begin{proof}[Proof of Theorem \ref{thm: Generic Rate}]
It follows immediately from Lemmas \ref{lem: Poincare} and \ref{lem: kernel variance} that
\begin{align*}
&\mathbb{E}_M[(\hat{P}_M (f) - f)^2] \\
&\quad\quad \leq  C \frac{\sqrt{\det M}}{n}\left(\frac{G_{2M}(p)}{G_M(p)^2}\right) + C \exp\left(-\frac{n}{\sqrt{\det M}}\left(\frac{G_M(p)}{10}\right)\right)+  C\left(1+\frac{\operatorname{Var}_{M_t}(p(Z))}{G_{M_t} (p)^2}\right)W(\mu_M)\\
&\quad\quad =O\left(\frac{\sqrt{\det M_t}}{n} + W(\mu_t)\right).
\end{align*}
\end{proof}

\subsubsection{General Kernels} \label{app: general kernels}

In this section, we discuss which of results above are applicable to kernels of the form $k(M^{1/2} x)$. We start with the following Poincar\'{e} inequality. Let $\mu_{k,M}$ denote the probability distribution with density $\propto k(M^{1/2} [x-x^*])$.

\begin{lemma}[Kernel Poincar\'{e}]
    The measure $\mu_{k,M}$ has covariance proportionate to $M^{-1}$. If $\mu_{k, I}$ satisfies a Poincar\'{e} inequality with constant $1/C_I$, then
    \[
    \operatorname{Var}_{\mu_{k,M}}(f) \leq \frac{1}{C_I}\int \nabla f^T M^{-1} \nabla f\, d\mu_{k,M}.
    \]
\end{lemma}
\begin{proof}
For convenience, we translate $x^* = 0$ without loss of generality. For the first claim, notice that when $M = I$, the density is symmetric about the origin, hence the covariance is proportionate to the identity. Thus the claim follows as $\mu_{k,M}$ is the pushforward of $\mu_{k,I}$ by $M^{-1/2}$.

Further,
\begin{align*}
   \operatorname{Var}_{\mu_{k,M}}(f) &=  \frac{1}{\sqrt{\det{M}}}\operatorname{Var}_{\mu_{k,I}}(f \circ M^{-1/2}) \leq \frac{1}{C_I} \int \frac{1}{\sqrt{\det M}}\| \nabla f(M^{-1/2}y) \|^2 d\mu_{k, I}(y).
\end{align*}
An application of the chain rule gives that
\begin{align*}
    \nabla_y f(M^{-1/2} y) = M^{-1/2} \nabla f(M^{-1/2}y).
\end{align*}
Hence,
\begin{align*}
    \operatorname{Var}_{\mu_{k,M}}(f) &\leq \frac{1}{C_I} \int \frac{1}{\sqrt{\det M}}\| M^{-1/2} \nabla f(M^{-1/2}y) \|^2\, \mu_{k,I}(y).
\end{align*}
Changing variable back to $\mu_{k,M}$ yields the desired result.
\end{proof}
With this in hand, the following is immediate for $\hat{P}_M$ with $k$ generic. The claim for the covariance similarly follows by change of variables, after observing that $\mu_{k,I}$ is rotation invariant, and therefore has isotropic covariance.

\begin{corollary}[Kernel MSE]
    Take $k$ to be nonnegative, and satisfying a Poincar\'{e} inequality for isotropic covariance. Assume that $f$ is bounded and $\varepsilon_i$ has finite fourth moments. Then there exists an absolute constant $C$ such that
\begin{align*}
    &\mathbb{E}\left[\int (\hat{P}_{M} (f) - f)^2\, d\mu_{k,M}\right] \leq\\
    &\quad\quad C \frac{\sqrt{\det M}}{n}\left(\frac{\mathbb{E}_{\mu_{k^2,2M}}[p]}{\mathbb{E}_{\mu_{k,2M}}[p]^2}\right) + C \exp\left(-\frac{n}{\sqrt{\det M}}\left(\frac{\mathbb{E}_{\mu_{k,2M}}[p]}{10}\right)\right) + C\left(1+\frac{\operatorname{Var}_{\mu_{k,M}}(p)}{\mathbb{E}_{\mu_{k,M}} [p]^2}\right)W(\mu_{k,M}).
\end{align*}
\end{corollary}

\subsection{Recursive Kernel Learning as Variance Minimization}

The proof of Proposition \ref{prop: basic alg} is clear from the main text, but we fill in the missing details here.

\begin{proof}[Proof of Proposition \ref{prop: basic alg}]
    The only thing left to clarify is the identity
    \[
    t A^{-1}/D = \operatorname{argmax}_{X \succeq 0, \operatorname{tr}(AX)\leq t} \log \det X .
    \]
    Note that $\log \det X$ is strictly convex, and KKT conditions yield that the minimum is a solution to the Lagrangian system of equations. Evaluating, we have
    \[
    L(X,\lambda) = \log \det X - \lambda (\operatorname{tr}(AX) - t) \Longrightarrow \nabla_X L = X^{-1} - \lambda A = 0.
    \]
    Hence the optimum is proportionate to $A^{-1}$, and evaluating the constraint yields the result.
\end{proof}

\subsection{Guarantees}

To derive the learning rate for the Local EGOP Learning algorithm, we mostly operate at the level of matrices, studying the recurrence from Equation \ref{eq: oracle egop}. To supplement this in the noisy manifold setting, we rely on critical geometric relations verified in Appendix \ref{app: geometric}.

\subsubsection{Sample Trimming}

A practical concern when implementing the Local EGOP Learning algorithm is that two points $p, p' \in \mathcal{M}$ may be such that $p-p' \in \mathcal{N}_p$, but $f(p) \neq f(p')$. Such a function can still be continuous index as the line segment between these points goes beyond the reach of the manifold, and thus outside the data sampling region. A simple fix is to exclude data points outside of a fixed isotopic radius of the point of interest $x^*,$ thus trimming the population distribution $P.$ The resulting EGOP matrix will then be defined with respect to an extension outside this domain. We adopt such an extension so that Lemma \ref{lem: grad rot}, and the secondary conclusions of Lemma \ref{lem: grad form}, that there is an affine space where the gradient is constant and the Hessian is degenerate, hold globally. Alternatively, one can broach this by taking the covariance cap to 0 at a logarithmic rate, reducing the sampling weight of the non-local region at the expense of a small variance inflation. For simplicity, we adopt the former assumptions.

\subsubsection{Supporting Arguments for Lemma \ref{lem: Taylor}} \label{app: geometric}

We try to keep our proofs as elementary as possible. We start with basic properties of our function of interest $f$.

\begin{lemma}[Derivatives]\label{lem: derivatives}
    Suppose that $f = f\circ \pi$. Then $\nabla f(x) \in \mathcal{T}_{x}$ and for any $u,v\in \mathcal{N}_{x}$, $ u^T \nabla^2 f(x) v = 0$.
\end{lemma}

\begin{proof}
    Let $u,v\in \mathcal{N}_{x}$ be arbitrary. Then,
    \begin{align*}
    \nabla f(x)^T v &= \lim_{t\to 0} [f(x + t v) - f(x)]/t = \lim_{t\to 0} [f(\pi[x + t v]) - f(\pi[x])]/t \\
    &= \lim_{t\to 0} [f(\pi[x]) - f(\pi[x])]/t = 0,
    \end{align*}
    as for small $t$ orthogonal perturbations do not change the nearest point on $\mathcal{M}$. As this holds for generic $v$ the first claim follows. To simplify notation, define $\partial_\eta h(x) := \lim_{t\to 0 } [h(x+t\eta)-h(x)]/t$.
    We can thus compute
    \begin{align*}
        u^T \nabla^2 f(x) v = \partial_{u} (\nabla f^T v)(x) = 0
    \end{align*}
    as $\nabla f^T v$ is identically 0 along the path $x+tv$, $t \in [0, \delta]$ for $\delta$ small enough. That is, as $\nabla f$ remains tangent along this path, as verified above, it is always orthogonal to $v$.
\end{proof}

In the interest of further studying the structure of $\nabla f$ we verify the following fundamental Geometric facts. We denote by $\pi_V$ the projection onto the vector space $V$. Denote by $d(\cdot, \cdot)$ the geodesic distance on $\mathcal{M}$. Let $\kappa = \sup_{p\in\mathcal{M}, \|v\|=1} \|S_v(p)\|$, $S$ the Weingarten map of $\mathcal{M}$. This constant is automatically finite as $\mathcal{M}$ has nontrivial reach.

\begin{lemma}[Orthogonal Curvature]\label{lem: geodesic}
    There exists a universal constant $C$ such that, for any $p,p^* \in \mathcal{M}$,
    \[
    \|\pi_{\mathcal{N}_p} - \pi_{\mathcal{N}_{p^*}}\| \leq C d(p, p^*).
    \]
\end{lemma}
\begin{proof}
    Let $\gamma_t$, $t\in [0, d(p, p^*)]$ be a unit speed geodesic connecting $p$ and $p^*$.  We can compute
    \[
    \|\pi_{\mathcal{N}_p} - \pi_{\mathcal{N}_{p^*}}\| = \left\|\int_0^1 \partial_t \pi_{\mathcal{N}_{\gamma_t}}\, dt  \right\| \leq \int_0^{d(p, p^*)} \|\partial_t \pi_{\mathcal{N}_{\gamma_t}}\|\, dt.
    \]
    Thus, if we can show a uniform bound for $\|\partial_t \pi_{\mathcal{N}_{\gamma_t}}\|$ then we are done.

    Let $v_1,\dots, v_{D-d}$ be an orthonormal basis of $\mathcal{N}_{x^*}$, and form the orthonormal moving frame $v_1(t),\dots, v_{D-d}(t)$ along $\gamma_t$ by parallel translation. Let
    \[
    V(t) :=
    \begin{bmatrix}
        v_1(t) & \dots & v_{D-d}(t)
    \end{bmatrix}.
    \]
    By construction, $\pi_{\mathcal{N}_{\gamma_t}} = V(t) V(t)^T$, $V(t)$ is orthogonal for all $t$. Differentiating,
    \[
    \|\partial_t \pi_{\mathcal{N}_{\gamma_t}}\| = \|V'(t) V(t)^T + V(t) V'(t)^T\| \leq 2 \|V'(t)\| \|V(t)\| = 2 \|V'(t)\|.
    \]
    This reduces the problem to uniformly bounding $\|V'(t)\|$. The key relations to prove this are
    \[
    V'(t) =
    \begin{bmatrix}
        v_1'(t) & \dots & v_{D-d}'(t)
    \end{bmatrix},\quad v_i'(t) = -S_{v_i(t)}(\gamma_t) \gamma'_t.
    \]
    The first identity is definitional, the second can be found in \cite{do2016differential}.
    For $\|u\| = 1$, we have
    \begin{align*}
        \|V'(t) u\| = \|\sum_{i=1}^{D-d} u_i v_i'(t)\| = \|-S_{\sum_i u_i v_i(t)}(\gamma_t) \gamma_t'\| \leq \| S_{\sum_i u_i v_i(t) }(\gamma_t)\gamma_t'\| \leq \kappa,
    \end{align*}
    as $\|\sum_i u_i v_i(t)\| \leq \sum_i u_i^2 = 1$, hence $\|V'(t)\| \leq \kappa$ holds uniformly.
\end{proof}

\begin{lemma}[Projection Normal]\label{lem: proj diff}
    For any $v \in \mathcal{N}_{\pi[x]}$, $D\pi_x v = 0$.
\end{lemma}

\begin{proof}
    As computed previously,
    \[
    D\pi_x v = \lim_{t\to 0} [\pi[x + tv] - \pi[x]]/t = 0.
    \]
\end{proof}

\begin{lemma}[Orthogonal Control]\label{lem: tan to orthog}
    Fix $x^*$ in the tubular neighborhood of $\mathcal{M}$. There exists $r>0$ small enough such that, for $\|x-x^*\| \leq r$, there exists a constant $C>0$ such that
    \[
    \|\pi_{\mathcal{N}_{x^*}} - \pi_{\mathcal{N}_{x}}\| \leq C \| \pi_{\mathcal{T}_{x^*}}[x^* - x]\|
    \]
\end{lemma}

\begin{proof}
    By Lemma \ref{lem: geodesic}, it suffices to show that $d(\pi[x^*], \pi[x]) = O(\| \pi_{\mathcal{T}_{x^*}}[x^* - x]\|).$ By \cite{garcia2019variational}, we have that $d(\pi[x^*], \pi[x]) = O(\| \pi[x^*] - \pi[x]\|)$, further reducing the problem. By our locality assumption, we can take convex combinations and remain within the tubular neighborhood, allowing us to decompose
    \[
    \|\pi[x^*] - \pi[x]\| \leq \left\|\int_0^1 D\pi_{x^* + t(x - x^*)}( x - x^*)\, dt   \right\| \leq \int_0^1 \|D\pi_{x^* + t(x - x^*)}( x - x^*)\|\, dt .
    \]
    By Lemma \ref{lem: proj diff},
    \begin{align*}
    &\|D\pi_{x^* + t(x - x^*)}( x - x^*)\| =  \|D\pi_{x^* + t(x - x^*)}(\pi_{\mathcal{T}_{x^* + t(x - x^*)}} [x - x^*])\| \\
    &\quad\quad \leq C \|\pi_{\mathcal{T}_{x^* + t(x - x^*)}}[x - x^*]\| \leq C\|[\pi_{\mathcal{T}_{x^* + t(x - x^*)}} - \pi_{\mathcal{T}_{x^*}}][x - x^*]\| + C\|\pi_{\mathcal{T}_{x^*}}[x - x^*]\| \\
    &\quad\quad \leq Cr\|\pi_{\mathcal{T}_{x^* + t(x - x^*)}} - \pi_{\mathcal{T}_{x^*}}\| + C\|\pi_{\mathcal{T}_{x^*}}[x - x^*]\|,
    \end{align*}
    as $\pi$ is smooth, and we have restricted ourselves to a bounded domain. In other words,
    \[
    \|\pi[x^*] - \pi[x]\| - Cr\|\pi_{\mathcal{T}_{x^* + t(x - x^*)}} - \pi_{\mathcal{T}_{x^*}}\| \leq C\|\pi_{\mathcal{T}_{x^*}}[x - x^*]\|,
    \]
    hence, if we can show that
    \[
    \|\pi_{\mathcal{T}_{x^* + t(x - x^*)}} - \pi_{\mathcal{T}_{x^*}}\| \leq C'  \|\pi[x^*] - \pi[x]\|,
    \]
    then for $r$ small enough the desired inequality will be achieved. Notice that $\pi_{\mathcal{T}_x} = I - \pi_{\mathcal{N}_x}$, hence by Lemma \ref{lem: geodesic}
    \[
    \|\pi_{\mathcal{T}_{x^* + t(x - x^*)}} - \pi_{\mathcal{T}_{x^*}}\| =
    \|\pi_{\mathcal{N}_{x^* + t(x - x^*)}} - \pi_{\mathcal{N}_{x^*}}\| \leq \Tilde{C} d(\pi[x^*], \pi[x^* + t(x - x^*)]) \leq C' \|\pi[x^*] - \pi[x]\|,
    \]
    completing the proof.
\end{proof}

We apply this to study $\nabla f$.

\begin{lemma}[Gradient Rotation]\label{lem: grad rot}
     Restricting ourselves to a region small enough that Lemma \ref{lem: tan to orthog} is valid, we have
     \[
     \|\pi_{\mathcal{N}_{x^*}}\nabla f(x)\| = O(\| \pi_{\mathcal{T}_{x^*}}[x^* - x]\| ).
     \]
\end{lemma}
\begin{proof}
     By Lemmas \ref{lem: derivatives} and \ref{lem: tan to orthog},
     \[
     \|\pi_{\mathcal{N}_{x^*}}\nabla f(x)\|  = \|[\pi_{\mathcal{N}_{x^*}} - \pi_{\mathcal{N}_{x}}]\nabla f(x)\| \leq C \|\pi_{\mathcal{N}_{x^*}} - \pi_{\mathcal{N}_{x}}\| \leq C'  \| \pi_{\mathcal{T}_{x^*}}[x^* - x]\|.
     \]
\end{proof}

\subsubsection{Supporting Arguments for Lemma \ref{lem: grad form}}\label{sec: Lemma 5}

\begin{lemma}\label{lem: proj diff_2}
     Let $x \in\mathcal{M}$ and $\eta \in \mathcal{N}_x$. Then $D\pi_{x+\eta} = (I - S_x(\eta))^{-1} \pi_{\mathcal{T}_x}$.
\end{lemma}

\begin{proof}
    Clearly, for any normal vector $v \in \mathcal{N}_x$, $D\pi_{x + \eta}v = 0$, as projection is invariant to normal displacement. Thus it suffices to verify the formula for any tangent vector. Our proof has two primary parts. First, we pass to normal coordinates, equating the tubular neighborhood about $\mathcal{M}$ with the normal bundle, and compute the differential of the exponential map. From this parameterization, we can then analyze the projection via the chain rule.

    Starting from the normal bundle $N\mathcal{M}$, for $(x,v) \in N_x \mathcal{M}$, the exponential map $E(x,v) = x + v$ is a diffeomorphism to the tube. Let $\gamma(t) = (u(t), \xi(t)) \in N \mathcal{M}$, $\gamma'(t)\in T N \mathcal{M}$. We compute $\xi'(0) = -S_{x}(v) u'(0) + \nabla_{u'(0)}^\perp \xi(0)$, hence
    \[
    DE_{(x,v)} \gamma'(0) = u + [-S_x(v)u + \nabla_{u'(0)}^\perp \xi(0)] =  (I - S_x(v))u + \nabla_{u'(0)}^\perp \xi(0).
    \]
    Composing with the projection map, $\pi \circ E(x,v) = x$, hence
    \[
    D\pi_{x+v} D E_{(x,v)} \gamma'(0) = u'(0) \quad \Longleftrightarrow \quad D \pi_{x+v}\pi_{\mathcal{T}_x} [(I - S_x(v) u]  = u'(0),
    \]
    and inverting yields the claim, as $D \pi_{x+v}$ is a linear operator.
\end{proof}

\begin{proof}[Proof of Lemma \ref{lem: grad form}]

The first claim follows immediately from Lemma \ref{lem: proj diff_2} and the chain rule, indeed,
\[
\nabla f(x+v) = \nabla g(\pi(x+v)) = (I-S_x(v))^{-1} \nabla g(x).
\]
Let $u = (I - S_x(v))^{-1} \nabla g(x)$. Applying linearity of the shape operator, observe that $(I - S_x(v + \eta)) u = (I - S_x(v))u - S_x(\eta)(u) = \nabla g(x) - S_x(\eta) u$. Hence, if $S_x(\eta) u = 0$ then $u = \nabla f(x+v + \eta)$. Thus, the second claim follows if we can show that there is a $D-2d$ dimensional subspace $\operatorname{Null}$ that annihilates the solution vector $u$. For $w \in \mathcal{T}_x$,
\[
\langle S_x(\eta) u, w \rangle = \langle \mathrm{I\!I}(u, w), \eta \rangle
\]
for $\mathrm{I\!I}$ the second fundamental form. This is a bilinear map, and the map defined by fixing the $u$ argument, $\phi_{u}(w) := \mathrm{I\!I}(u, w)$ is a linear map from $\mathcal{T}_x \to \mathcal{N}_x$, and thus has image no larger than $d$. Its complement $\phi_u(\mathcal{T}_x)^\perp$ has dimension at least $\operatorname{dim}(\mathcal{N}_x)-d = D - 2d$, as desired.

This proves the invariance of the gradient along a space of dimension at least $D-2d$, which immediately implies $\nabla^2 f(x + v + \eta) z = \partial_z \nabla f(x + v) = 0.$
\end{proof}

\subsubsection{Taylor Expansion}

We first provide the generic Taylor expansion, which we then go on to refine for the noisy manifold setting.

\begin{proof}[Proof of Lemma \ref{lem: Generic Taylor}]
    We start by Taylor expanding the gradient. For a $D$ tuple $\xi$, as $\partial_{\xi}\partial_{i}f = \partial_i \partial_{\xi}f$, for $f$ sufficiently smooth we can express
    \[
    \partial_i f(x) = \sum_{|\xi| \leq q} \frac{1}{q!}\partial_i  [\partial_\xi f](x) \prod_{q_i\in \xi} x_i^{q_i} + O(\|x\|^{q+1}).
    \]
    In particular, we have
    \[
    \partial_i f(x) = \partial_i f(0) + \sum_{j} \partial_i \partial_j f(0) x_j + \frac{1}{2}\sum_{j,k} \partial_i \partial_j \partial_k f(0) x_j x_k + O(\|x\|^3).
    \]
    Working at the level of the gradient, this yields
    \[
    \nabla f(x) = \nabla f(0) + \nabla^2 f(0) x + T(x, x) + O(\|X\|^3)
    \]
    where $T$ is the bilinear for
    \[
    [T[x,y]]_i = \frac{1}{2}\sum_{j,k} \partial_i \partial_j \partial_k f(0) x_j y_k.
    \]
    We use the first three terms of this Taylor expansion as a starting point. Set $g:= \nabla f(0), H:= \nabla^2 f(0)$.

    The remainder of the Taylor expansion can be exactly expressed as
    \[
    \Tilde{R}(x) = \frac{1}{6} \int_0^1 (1-s)^2 \nabla^4 f(sx) [x,x,x]\, ds,\quad [\nabla^4 f(sx) [x,x,x]]_i:= \sum_{j,k,\ell} \partial_i \partial_j \partial_k \partial_\ell f(sx)x_i x_j x_k.
    \]
    Hence, under our uniform boundedness assumption, we have
    \[
    \|T[x,x]\| \leq C \|x\|^2,\quad \|\Tilde{R}(x)\| \leq C \|x\|^3
    \]
    for some constant $C$. Throughout the argument, for convenience we will absorb constants into $C$ if necessary. Hence,
    \[
    \mathcal{L}(\mu) = \int \nabla f \nabla f^T\, d\mu = gg^T + H\Sigma H + gT(\Sigma)^T + T(\Sigma) g^T + R(\Sigma),
    \]
    for
    \begin{align*}
    T(\Sigma)_i &:= \mathbb{E}_\mu[T(x,x)_i] = \frac{1}{2}\sum_{j,k} \partial_i \partial_j \partial_k f(0) \Sigma_{jk}\\
    R(\Sigma)&:=  \mathbb{E}_\mu[T[Z,Z] T[Z,Z]^T + HZ \Tilde{R}(Z)^T + \Tilde{R}(Z) (HZ)^T + \Tilde{R}(Z) \Tilde{R}(Z)^T].
    \end{align*}
    That $T$ is continuous and linear in $\Sigma$ is immediate. For the remainder, we have
    \begin{align*}
    \|R(\Sigma)\| &\leq  \mathbb{E}_\mu[\|T[Z,Z] T[Z,Z]^T\| + \|HZ \Tilde{R}(Z)^T\| + \|\Tilde{R}(Z) (HZ)^T\| + \|\Tilde{R}(Z) \Tilde{R}(Z)^T\|]\\
    &\leq C\mathbb{E}[\|Z\|^4 + \|Z\|^6] \leq C \mathbb{E}_\mu[\operatorname{tr}(\Sigma^2) + \operatorname{tr}(\Sigma)^2]
    \end{align*}
    where we apply Isserlis's Theorem \cite{isserlis1918formula} which verifies the vanishing of third order terms, and relates the higher moments to the norm of the covariance matrix. Our capped covariance assumption is used to compress all higher order terms to quadratic order. Our proof is completed upon observing
    \[
    \operatorname{tr}(\Sigma^2) = \|\Sigma\|_F^2,\quad \operatorname{tr}(\Sigma)^2 = \left(\sum_i \lambda_i\right)^2 \leq D \sum_i \lambda_i^2 = D \|\Sigma\|_F^2,
    \]
    and in finite dimensions the frobenius and operator norms are equivalent.
\end{proof}

\begin{proof}[Proof of Lemma \ref{lem: Taylor}]
    That $\operatorname{rank}(H)\leq 2d$ is immediate from Lemma \ref{lem: derivatives}. Indeed, the $\mathcal{N}\times \mathcal{N}$ block of $H$ is completely 0, hence at most $d$ rows and $d$ columns of $H$ in the $\mathcal{T}\oplus \mathcal{N}$ basis have nonzero entries. Let $B_r$ be

    We now apply Lemmas \ref{lem: derivatives} and \ref{lem: grad rot} to compute
    \begin{align*}
       &\|\pi_{\mathcal{N}} \mathcal{L}(\mu) \pi_{\mathcal{N}}\| = \left\|\int \pi_{\mathcal{N}}\nabla f\nabla f^T \pi_{\mathcal{N}}\, d\mu\right\| \leq \int\| \pi_{\mathcal{N}}\nabla f\nabla f^T \pi_{\mathcal{N}}\|\, d\mu \\
       &\quad\quad = \int\| [\pi_{\mathcal{N}} - \pi_{\mathcal{N}_x}]\nabla f(x)\nabla f(x)^T [\pi_{\mathcal{N}} - \pi_{\mathcal{N}_x}]\|\, d\mu\\
       &\quad\quad \leq C \int \|\pi_{\mathcal{N}} - \pi_{\mathcal{N}_x}\|^2\, d\mu \leq C \int \|\pi_{\mathcal{T}} x\|^2\, d\mu \leq C [\operatorname{tr}([\pi_{\mathcal{T}}\Sigma\pi_{\mathcal{T}}]^2) + \operatorname{tr}(\pi_{\mathcal{T}} \Sigma \pi_{\mathcal{T}})^2] \leq C \|\pi_{\mathcal{T}} \Sigma \pi_{\mathcal{T}}\|_F^2,
    \end{align*}
    where the final bounds from Isserlis's Theorem are similar to Lemma \ref{lem: Generic Taylor}. Similarly,
    \begin{align*}
       &\|\pi_{\mathcal{T}} \mathcal{L}(\mu) \pi_{\mathcal{N}}\|  \leq \int\| \pi_{\mathcal{T}}\nabla f(x)\nabla f(x)^T [\pi_{\mathcal{N}} - \pi_{\mathcal{N}_x}]\|\, d\mu \leq C \int \|\pi_{\mathcal{T}} x\|\, d\mu \\
       &\quad\quad \leq C \sqrt{\int \|\pi_{\mathcal{T}} x\|^2\, d\mu} = C \|\pi_{\mathcal{T}} \Sigma \pi_{\mathcal{T}}\|_F.
    \end{align*}
    As the operator and Frobenius norms are equivalent the bounds are as desired.
\end{proof}

\begin{proof}[Proof of Corollary \ref{cor: shifted taylor}]
    The basic claim is immediate from Lemma \ref{lem: Generic Taylor}, and that the terms stemming from higher order derivatives are uniformly bounded follows from the smoothness assumptions on $f$.
    Further, the nullspace of the Hessian and the 0 order expansion being the gradient $g = \nabla f(x^*)$ follow immediately from Lemma \ref{lem: grad form}.
\end{proof}

\subsubsection{Matrix Recurrence}\label{sec: recurrence}

We begin our study of the recurrence in Equation \ref{eq: oracle egop} by introducing a reduced form. We analyze this scheme, and then apply it to Local EGOP Learning, leveraging the results of Appendix \ref{app: geometric}. We divide our focus into two distinct settings, the first mirroring Equation \ref{eq: constant part} by introducing a constant shift, and the second mirroring Equation \ref{eq: homogeneous} a homogeneous equation.

\textbf{Constant Shift}

Let $J$ be an involution, $J = J^T = J^{-1}.$
We introduce the recurrence
\[
B_{t+1} = h_{t+1}J(G + \beta B_t + (1-\beta) B_{t-1} + E_t)^{-1}J
\]
where $E_t$ is an error term, $\beta \in (0, 1)$,  $G\succ 0$,  $B_0 = \alpha I$, $h_1 = \alpha^2$, $B_1 = h_1 J(G+B_0)^{-1}J$, $h_{t+1} = t_i r^2$, $r<1$.

\begin{lemma}[Perturbed Shift]\label{lem: shift}
    Suppose there exists a uniform constant $C$ such that $\|E_t\| \leq C\alpha$ for all $t$. Then, for $\alpha$ small enough,
    \[
    \|B_{t+1}\| = \Theta(h_t).
    \]
\end{lemma}
\begin{proof}
    That $\alpha$ must be small is not necessary, but this will be a typical setting and will help expedite the proof.
    Let $\alpha<1$ be small enough that $\|B_t\| \leq \lambda_{\min}(G)/4$, $t=1,2$, and $C\alpha \leq \lambda_{\min}(G)/4$. We prove by strong induction that $\|B_{t+1}\|\leq h_{t+1} \frac{2}{\lambda_{\min}(G)}$. The base case is immediate, so we move on to the inductive step. We can simply compute
    \[
    G + \beta B_t + (1-\beta) B_{t-1} + E_t \succeq G/2.
    \]
    As $J$ is an orthogonal matrix, we have $ \|B_{t+1}\| \leq 2 h_{t+1}/\lambda_{\min}(G)$ and the claim follows.

    With the decay of $B_t$ being settled, it follows that $B_t= h_{t+1} JG^{-1}J + O(h_{t+1}^2 + h_{t+1}\alpha)$, hence, normalizing by $h_{t+1}$, we have exact convergence to $JG^{-1}J$ in the noiseless setting.
\end{proof}

\textbf{Homogeneous}

We move on to the recurrence
\[
A_{t+1} = h_{t+1}J(\beta A_t + (1-\beta) A_{t-1} + E_t)^{-1}J
\]
where $E_t$ is an error term, $J$ is an involution, $\beta \in (0, 1)$,  $G\succ 0$,  $B_0 = \alpha I$, $h_1 = \alpha^2$, $B_1 = h_1 J(G+B_0)^{-1}J$, $h_{t+1} = h_t r^2$, $r<1$.

We start by reducing this iteration scheme to an autonomous equation.

\begin{proposition}[Autonomous Recurrence]\label{prop: autonomous}
    Define $c(r, \beta) := \frac{1}{\beta/r + (1-\beta)/r^2}$ and let $X_t := A_t/\sqrt{c(r,\beta) h_t}$, $R_t := \sqrt{c(r,\beta)} E_t/\sqrt{h_{t+1}}$, and $\eta = \frac{\beta/r}{c(r,\beta)}$. Then $X_t$ satisfies the autonomous recurrence
    \[
    X_{t+1} = J( \eta X_t + (1-\eta) X_{t-1} + R_t)^{-1}J
    \]
\end{proposition}
\begin{proof}
    We work this out by simple algebra. First, as an intermediate step, introduce $C_t := A_t/ \sqrt{h_{t}}$, $R_t' := E_t/\sqrt{h_{t+1}}$. Then we can write
    \begin{align*}
        &A_{t+1} = h_{t+1}J(\beta A_t + (1-\beta) A_{t-1} + E_i)^{-1}J \\
        &\quad\quad \Longleftrightarrow \frac{A_{t+1}}{\sqrt{h_{t+1}}} = J\left(\beta \frac{A_t}{\sqrt{h_{t+1}}} + (1-\beta) \frac{A_{t-1}}{\sqrt{h_{t+1}}} + \frac{E_t}{\sqrt{h_{t+1}}}\right)^{-1}J\\
        &\quad\quad \Longleftrightarrow C_{t+1} = J([\beta/r] C_t + [(1-\beta)/ r^2] C_{t-1} + R_t')^{-1}J \\
        &\quad\quad \Longleftrightarrow C_{t+1} = c(r,\beta) J(\eta C_t + (1-\eta) C_{t-1} + R_t' c(r,\beta))^{-1}J\\
        &\quad\quad \Longleftrightarrow \frac{C_{t+1}}{\sqrt{c(r,\beta)}} = J\left(\eta \frac{C_t}{\sqrt{c(r,\beta)}} + (1-\eta) \frac{C_t}{\sqrt{c(r,\beta)}} +  R_t' \sqrt{c(r,\beta)}\right)J \\
        &\quad\quad \Longleftrightarrow X_{t+1} = J(\eta X_t + (1-\eta)X_{t-1} + R_t)^{-1}J.
    \end{align*}
\end{proof}

Thus to study $A_t$ it suffices to study $X_t$, then multiply by the corresponding factors.
Our analysis proceeds as follows. First, leveraging the Thompson metric \citep{thompson1963certain}, we show that, for summable error $R_t$, the spectrum of $X_t$ stays bounded away from 0 and $\infty$. Then, utilizing this bound, we go on to show convergence of $X_t$.
Define the Thompson metric,
    \[
    d_T(A, B) := \log \max\{M(A/B), M(B/A)\},\quad M(A/B):= \inf \{\gamma >0 : A \preceq \gamma B\}.
    \]

\begin{lemma}[Spectral Band]\label{lem: spectral band}
    Suppose that $\sum_i \|R_i\| \leq C\alpha$ and $\max\{d_T(X_0, I), d_T(X_1, I)\} = L_0$. Then, for $\alpha$ sufficiently small,
    \[
    [e^{-L_0} - O(\alpha)]I \preceq X_t \preceq [e^{L_0} + O(\alpha)] I
    \]
    for all $i$.
\end{lemma}

\begin{proof}
    Notice that
    \[
    (1- \xi_t)[\eta X_t + (1-\eta) X_{t-1}] \preceq \eta X_t + (1-\eta) X_{t-1} + R_t \preceq (1+\xi_t) [\eta X_t + (1-\eta) X_{t-1}]
    \]
    for $\xi_t := \| [\eta X_t + (1-\eta) X_{t-1}]^{-1/2} R_t [\eta X_t + (1-\eta) X_{t-1}]^{-1/2}\|.$ It follows that
    \[
    d_T([\eta X_t + (1-\eta) X_{t-1}] + R_t, [\eta X_t + (1-\eta) X_{t-1}]) \leq \log \frac{1+\xi_t}{1-\xi_t}
    \]
    Hence, if $e^{-L_t}I\preceq X_t, X_{t-1} \preceq e^{L_t} I$, it follows that $\xi_t \leq e^{L_t} \|R_t\|$ and
    \[
    L_{t+1} \leq L_t + \log \frac{1 + \|R_t\| e^{L_t}}{1-\|R_t\| e^{L_t}} \Longleftrightarrow e^{-L_{t+1}} \geq e^{-L_t} \left(\frac{e^{-L_t} - \|R_t\|}{e^{-L_t}+ \|R_t\|}\right)\Longrightarrow e^{-L_{t+1}} \geq e^{-L_t} - 2\|R_t\|.
    \]
    Thus,
    \[
    e^{-L_{t}} \geq e^{-L_0} - 2 \sum_{i=1}^\infty \|R_t\| \geq e^{-L_0}/2
    \]
    choosing $\alpha$ sufficiently small.
\end{proof}

\begin{lemma}[Non-Expansive]\label{lem: non-expansive}
    Let $X_t$, $X_t'$ be autonomous homogeneous recurrences of the form of Proposition \ref{prop: autonomous}, with common conjugation matrix $J$. Then
    \[
    d_T(X_{t+1}, X_{t+1}') \leq \max\{d_T(X_{t}, X_{t}'), d_T(X_{t-1}, X_{t-1}')\} + O(\|R_t\| + \|R_t'\|)
    \]
\end{lemma}

\begin{proof}
    Set $M_t = \eta X_t + (1-\eta) X_{t-1}, M_t' = \eta X_t' + (1-\eta) X_{t-1}'$.
    Using the invariance of the Thompson metric to conjugate by orthogonal matrices and invert, we can compute
    \begin{align*}
        d_T(X_{t+1}, X_{t+1}')&= d_T(\eta X_t + (1-\eta) X_{t-1} + R_t, \eta X_t' + (1-\eta) X_{t-1}' + R_t') \\
        &\leq d_T(M_t, M_t') + d_T(M_t, M_t + R_t) + d_T(M_t', M_t' + R_t').
    \end{align*}
    Checking each term,
    \[
    d_T(M_t, M_t') \leq \max\{d_T(X_{t}, X_{t}'), d_T(X_{t-1}, X_{t-1}')\}
    \]
    as averaging in the Thompson metric is non-expansive.
    \[
    d_T(A, A + A) = \|\log A^{-1/2} (A + E) A^{-1/2} \| = \| \log (I + A^{-1/2} E A^{-1/2})\| = O(\|E\|)
    \]
    for $A$ with non-vanishing spectrum. Lemma \ref{lem: spectral band} verifies that the sequences $X_t,X_t'$ do not degenerate for $\alpha$ sufficiently small, thus showing the claim.
\end{proof}

\begin{lemma}[Definite Hessian]\label{lem: convergence}
    Let $J = I$. Then for $\alpha$ sufficiently small, $X_t \to I.$
\end{lemma}

\begin{proof}
    Define $T(x,y) = ((\eta x + (1-\eta)y)^{-1}, x)$, hence $T(X_t, X_{t-1}) = (X_{t+1}, X_t)$. Define
    \[
    V(x,y) = \max \{d_T(x,I), d_T(y,I)\} = \max \{|\log \lambda_{\max}(x)|, |\log \lambda_{\max}(y)|, |\log \lambda_{\min}(x)|, |\log \lambda_{\min}(y)|\}.
    \]
    Notice that
    \[
    \min\{\lambda_{\min}(x), \lambda_{\min}(y)\} \leq \lambda_{\min}(\eta x + (1-\eta)y)^{-1} \leq \lambda_{\max}(\eta x + (1-\eta)y)^{-1} \leq \max\{\lambda_{\max}(x), \lambda_{\max}(y)\}
    \]
    hence $V(T(X_{t}, X_{t-1})) \leq V(X_t, X_{t-1})$. Thus the noiseless recurrence is autonomous, and the discrete Lasalle invariance principle \citet[Theorem 2]{mei2017lasalle} yields the existence of a set $E_c$, $T(E_c) \subseteq E_c$ such that for all $(X,Y)\in E_c$, $V(E_c) = c$ and $\lim_{t\to \infty} \inf_{E\in E_c} \|(X_t, X_{t-1}) - E\| \to 0$. To verify that $X_t \to I$ in the noiseless case, it suffices to show $E_0 = \{(I,I)\}$ is the only invariant set with constant value function.

    We argue by contradiction. Let $E_c$ be such a limit set, $c> 0$. Take $(X_1, X_0) \in E_c$, and consider $(X_3, X_2) = T^2(X_1, X_0)$. By the condition on the value function $V$, either the largest eigenvalue of these matrices is $e^c$ or the smallest is $e^{-c}$. Without loss of generality, assume that $X_3$ achieves the eigenvalue $e^c$.

    For this to occur, it must be that both $X_2, X_1$ have eigenvalue $e^{-c}$ with a common eigenvector $v$. As the sequence is rational in $X_2, X_1$, $v$ will be an eigenvector for all subsequent matrices $X_t$. Notice $\log \lambda_v(X_4) = -\log (\eta e^c + (1-\eta) e^{-c}) \in (-c, c)$, hence $\sup_{t>0}|\log \lambda_v(X_{3+t})|<c$, and this eigenspace can never again be extremal.
    Indeed, for $\lambda_v(X_{3+t})=e^{\pm c}$, we require $\lambda_v(X_{2+t}) =  \lambda_v(X_{1+t}) = e^{\mp c}$,
    thus any extremal eigenspace for $T^2(X_1,X_0)$ will decay upon further iterations, and non-extremal eigenspaces cannot become extremal. In particular, $V(T^k(X_1, X_0))$ must decrease as $k\to \infty$, contradicting that $E_c$ is invariant with fixed value function $V(E) = c$ for all $E\in E_c$.

    In the noisy setting, Lemma \ref{lem: spectral band} verifies that the recurrence remains compact for $\alpha$ sufficiently small. Hence, we can apply \citet[Theorem 1.8]{mischaikow1995asymptotically} to verify that the perturbed recurrence converges to the same limit sets as the autonomous recurrence, i.e. $I$ is the unique limit point.
\end{proof}

\begin{lemma}[Generic J]\label{lem: generic J}
    Suppose that $J$ has signature $(k, j)$. Then, for $\alpha$ sufficiently small, $X_i \to L\in \{X: (XJ)^2 = I\}$. Further, $L$ has at most $2 \min\{k, j\}$ eigenvalues not equal to 1.
\end{lemma}

\begin{proof}
    We update the recurrence in Lemma \ref{lem: convergence} by $J$ congruence. $V$ is still monotone in the autonomous noiseless recurrence as $J$ is orthogonal. We begin by investigating the structure of any invariant set $E_c$.

    As in the previous proof, if we initialize with a pair in $E_c$, then for any $t \geq 2$, and all extremal eigenpairs $(\lambda_\ell, v_\ell)$, it must be that for both $j=t-1,t-2$, the matrices $X_j$ have matching eigenpairs $(1/\lambda_\ell, Jv_\ell)$. Repeating this argument for $X_{t+1}$, for any extremal eigenpairs $(\nu_\ell, u_\ell)$ it follows that $X_t, X_{t-1}$ have matching eigenpairs $(1/\nu_{\ell}, Ju_{\ell})$. Combining this information, for each extremal eigenpair $(\nu_\ell, u_{\ell})$ of $X_{t+1}$, $(1/\nu_\ell, J u_{\ell})$ must be an eigenpair of both $X_t$ and $X_{t-1}$. As $(\lambda_\ell, v_\ell) = (1/\nu_{\ell},  Ju_{\ell})$ is an extremal eiegenpair of $X_t$, it must also be that $(\nu_{\ell}, u_{\ell})$ is an eigenpair of $X_{t-1}$. Hence the matrix $X_{t-1}$ has both $(1/\nu_\ell, J u_\ell)$ and $(\nu_{\ell}, u_{\ell})$ as extremal eigenpairs, and for them to pass forward in the recurrence they must be shared by $X_{t-1}.$ Thus we see that for all $m \geq t$, $(1/\nu_\ell, J u_\ell)$ and $(\nu_{\ell}, u_{\ell})$ are fixed eigenvectors of $X_m$.

    Thus we have proven that any limiting set in $E_c$ has fixed, cyclical extremal eigenpairs. Suppose that $X_t$ has $\kappa$ such limiting maximal eigenvalues. Let $\operatorname{Trunc}_{\kappa}(X)$ be the map that identifies $X$ with the $(D - \kappa) \times (D - \kappa)$ matrix formed by removing the principal $\kappa$ components.
    Thus, $\operatorname{Trunc}_{\kappa}T(X_{t+1}, X_t) = T(\operatorname{Trunc}_{\kappa}(X_{t}, X_{t-1})) + o(1)$, and the sequence $\operatorname{Trunk}_\kappa(X_t, X_{t-1})$ is asymptotically autonomous and monotone in the value function $V$. Hence \citet[Theorem 1.8]{mischaikow1995asymptotically} yields that the truncated sequence has asymptotically cyclical extremal eigenvectors.

    Inductively applying this argument verifies that $X_t \to L$ for a fixed matrix $L$ in the noiseless case. By definition
    \[
     L = \lim_{t\to \infty} X_{t+1} = \lim_{t\to \infty} J(\eta X_{t+1} + (1-\eta) X_t)^{-1} J = JL^{-1}J,
    \]
    hence $L \in \{X : (LJ)^2 = I\}.$ In the perturbed setting, Lemma \ref{lem: spectral band} verifies that the recurrence remains compact for $\alpha$ sufficiently small. Hence, we can once again apply \citet[Theorem 1.8]{mischaikow1995asymptotically} to verify that the recurrence converges to the same limit sets as the autonomous recurrence.

    We now study the structure of such $L$. As we demonstrated, every eigenpair $(\lambda, v)$ of $L$ must have matching eigenpair $(1/\lambda, Jv)$. Let $v_1,\dots, v_\kappa, Jv_1, \dots, J v_\kappa, u_1,\dots, u_{D-2\kappa}$, where $Ju_i = \pm u_i$ are invariant under $J$. Displaying $J$ in this basis, we have
    \[
    J =
    \begin{bmatrix}
        0 & I & 0\\
        I & 0 & 0\\
        0 & 0 & D
    \end{bmatrix}
    \]
    The first $2\kappa \times 2\kappa$ block has signature $(\kappa, \kappa)$, yielding the desired claim.
\end{proof}

\begin{lemma}[Commuting Seeds]\label{lem: second order}
    Let $r \in (0, 1)$, $\|R_t\| \leq C\alpha r^t$, $X_0, X_1$ commute with $J$. Then, for any $\varepsilon>0$, there exists $\delta > 0$ such that $X_t \to L_{\alpha}$, $\|L_\alpha - I\| < \varepsilon$.
\end{lemma}

\begin{proof}
    That $X_t$ has a limit $L_\alpha$ for $\alpha$ small follows immediately from Lemma \ref{lem: generic J}.
    Thus it suffices to verify that $\lim_{t\to \infty} X_i = I$ for the unperturbed sequence. Indeed, if for $t$ large enough, $\|X_{t} - I\| < \varepsilon, \|X_{t-1} - I\| < \varepsilon$, because $X_t$ is rational in $X_0, X_1, R_j$, $j = 1,\dots, i-1$, it follows that it is smooth for $\alpha$ small, thus incurring additional $O(\alpha)$ error. In combination with Lemma \ref{lem: spectral band}, this verifies the claim.

    Without perturbations, the whole sequence commutes, and
    \[
    J( \eta X_t + (1-\eta) X_{t-1})^{-1}J = ( \eta X_t + (1-\eta) X_{t-1})^{-1} J^2 = ( \eta X_t + (1-\eta) X_{t-1})^{-1}
    \]
    hence $J$ factors out, and the convergence is a consequence of Lemma \ref{lem: convergence}.
\end{proof}

\subsubsection{Full Rank Hessian}
Our main proof strategy is to reduce the Local EGOP iterations to the form of Proposition \ref{prop: autonomous}, so that we can guarantee stable first and second order convergence rates by Lemmas \ref{lem: shift} and \ref{lem: generic J}. Let $H$ have diagonalization $H= U \Lambda U^T$. A straightforward change of variables yields the following.

\begin{proposition}[H Reduction]\label{prop: H reduction}
    Let $J = \operatorname{sign}(\Lambda)$ be the diagonal signature matrix and $W = U |\Lambda|^{1/2}$.

    Define the transformation
    \begin{gather*}
        X_t := W^T \Sigma_t W, \qquad \tilde{g} := J W^{-1} g, \\
        \tilde{T}(X) := J W^{-1} T(W^{-T} X W^{-1}) W^{-T} J, \qquad \tilde{R}(X) := J W^{-1} R(W^{-T} X W^{-1}) W^{-T} J.
    \end{gather*}
    Then the recurrence for $X_t$ depends only on $J$:
    \begin{align*}
    X_{t+1} = t_{t+1} J \Big( &\tilde{g} \tilde{g}^T + \beta X_t + (1-\beta)X_{t-1} \\
    &+ \beta [ \tilde{g} \tilde{T}(X_t)^T + \tilde{T}(X_t) \tilde{g}^T + \tilde{R}(X_t) ] \\
    &+ (1-\beta)[ \tilde{g} \tilde{T}(X_{t-1})^T + \tilde{T}(X_{t-1}) \tilde{g}^T + \tilde{R}(X_{t-1}) ] \Big)^{-1} J
    \end{align*}
\end{proposition}

Thus we can reduce our recurrence to that studied in Section \ref{sec: recurrence}.

\begin{lemma}[Schur Complement]\label{lem: schur}
    Define $u := g/\|g\|$, $\ \bar\pi := u u^T$, $\ \bar Q:= I - \bar\pi$, $H_{\bar Q} = \bar Q H \bar Q$. Let
    \begin{gather*}
    % Line 1: Pack all simple scalars and vectors
    \alpha_t := u^T \Sigma_t u, \quad \alpha_t' := u^T H\Sigma_tH u, \quad \eta^2 := \|g\|^2, \quad b_t := \bar Q \Sigma_t u, \quad b_t' := \bar Q H \Sigma_t H u, \\
    % Line 2: Simple matrices and intermediate R terms
    C_t := \bar Q \Sigma_t \bar Q, \quad C_t' := \bar Q H \Sigma_t H \bar Q, \quad R_{b,t} := \eta \bar Q T(\Sigma_t) + \bar Q R(\Sigma_t) u,\\
    % Line 3: Long scalar update
    \bar R_{\alpha, t} := \beta[2 \eta u^T T(\Sigma_t) + u^T R(\Sigma_t) u] + (1-\beta)[2 \eta u^T T(\Sigma_{t-1}) + u^T R(\Sigma_{t-1}) u],\\
    % Line 4: Medium updates packed together (b and R_C)
    \bar b_t = \beta[b_t' + R_{b,t}] + (1-\beta) [b_{t-1}' + R_{b,t-1}], \qquad \bar R_{C,t} := \beta\bar Q R(\Sigma_t)\bar Q + (1-\beta) \bar Q R(\Sigma_{t-1})\bar Q,\\
    % Line 5: Final Matrix update
    \bar S_t := \beta H_{\bar Q} C_t H_{\bar Q} + (1-\beta) H_{\bar Q}C_{t-1} H_{\bar Q} + \bar R_{C,t} - \frac{1}{\eta^2 + \alpha_t + \bar R_{\alpha, t}}\bar b_t \bar b_t^T.
\end{gather*}
    Then, in the orthonormal basis $\{u\}\oplus \mathrm{range}(\bar Q)$,
    \begin{gather*}
    \begin{bmatrix}
        \alpha_{t+1} & b_{t+1}^T\\
        b_{t+1} & C_{t+1}
    \end{bmatrix}
    = t_{t+1}\\
    \begin{bmatrix}
        \displaystyle \frac{1}{\eta^2 + \alpha_t' + \bar R_{\alpha, t} } +  \left(\frac{1}{\eta^2 + \alpha_t' + \bar R_{\alpha, t} }\right)^2 \bar b_t^T \bar S_t^{-1} \bar b_t
        &
        \displaystyle - \frac{1}{\eta^2 + \alpha_t' + \bar R_{\alpha, t} }  \bar b_t ^T \bar S_t^{-1} \\
        \displaystyle - \bar S_t^{-1} \bar b_t  \frac{1}{\eta^2 + \alpha_t' + \bar R_{\alpha, t} }
        &
        \displaystyle \Big(\beta C_t' + (1-\beta)  C_{t-1}' + \bar R_{C,t} - \frac{1}{\eta^2 + \alpha_t + \bar R_{\alpha, t}}\bar b_t \bar b_t^T\Big)^{-1}\!
    \end{bmatrix}.
    \end{gather*}
\end{lemma}

\begin{proof}
    This is the Schur complement with pivot on the top-left scalar applied to $M_t := [\beta \mathcal{L}(\mu_t) + (1-\beta)\mathcal{L}(\mu_{t-1})] ^{-1}$.
\end{proof}

\begin{corollary}[Base Case]\label{cor: base case}
    Suppose that $\Sigma_0 = \alpha I$ and $t_1 = \alpha^2$. Then, in the $(u,\bar Q)$ basis,
    \[
    \Sigma_1 =
    \begin{bmatrix}
        \frac{\alpha^2}{\eta^2} + O(\alpha^3) & O(\alpha^2)\\
        O(\alpha^2) & \alpha (\bar Q H^2 \bar Q)^{-1} + O(\alpha^2)
    \end{bmatrix},\quad
    \Sigma_2 =
    \begin{bmatrix}
        \frac{\alpha^2}{\eta^2} + O(\alpha^3) & O(\alpha^2)\\
        O(\alpha^2) & \alpha H_{\bar Q}^{-1} H^2 H_{\bar Q}^{-1} + O(\alpha^2)
    \end{bmatrix}
    \]
\end{corollary}

\begin{proof}
    This can be directly computed from the formula in Lemma \ref{lem: schur}. As a subtle detail, at the first step, the covariance is isotropic, hence in the $\bar Q \times \bar Q$ block the dominant terms is $\alpha( \bar Q H (I) H \bar Q)^{-1} = \alpha ( \bar Q  H^2 \bar Q)^{-1} $, however, after this step the gradient element of the covariance is a higher order residual, being of the Lemma \ref{lem: shift} constant shift form, and thus subsequently $\bar Q H \Sigma_1 H \bar Q = (\bar Q H \bar Q) \Sigma_1 (\bar Q \Sigma_1 \bar Q) + O(\alpha^2) =: H_{\bar Q} \Sigma_1 H_{\bar Q} + O(\alpha^2).$
\end{proof}

\begin{proof}[Proof of Theorem \ref{thm: Generic Rate}]
     We seek to inductively apply Lemmas \ref{lem: shift} and \ref{lem: second order} to the $u\times u$ and $\bar Q \times \bar Q$ blocks respectively, necessitating a careful tabulation of the residuals and off-diagonals. Our inductive hypothesis is that
     \begin{align*}
         \alpha_t = \Theta(h_t),\quad b_t = O(\sqrt{h_t}),\quad C_t = \Theta(\sqrt{h_t}), \quad \bar R_{C,t} = O(h_t),\quad \bar R_{b,t} = O(\sqrt{h_t}),\quad \bar R_{\alpha,t} = O(\sqrt{h_t}).
     \end{align*}
     We argue by strong induction, with the base case being covered in Corollary \ref{cor: base case}. We argue one block at a time. First note that $\bar b_t^T \bar S_t^{-1} \bar b_t = O(\sqrt{h_t})$, hence
     \[
     \frac{1}{\eta^2 + \alpha_t' + \bar R_{\alpha, t} } +  \left(\frac{1}{\eta^2 + \alpha_t' + \bar R_{\alpha, t} }\right)^2 \bar b_t^T \bar S_t^{-1} \bar b_t = \frac{1}{\eta^2} + O(\sqrt{h_t}) \Longrightarrow \alpha_{t+1} = \frac{h_{t+1}}{\eta^2} + O\left(h_{t+1}^{3/2}\right).
     \]
     Thus, taking $\alpha$ sufficiently small, we get a uniform $O(h_t)$ bound by Lemma \ref{lem: shift} and Proposition \ref{prop: H reduction}.

     For the $\bar Q \times \bar Q$ block,
     \[
     C_{t+1}
     =
     h_{t+1}\Big(\beta H_{\bar Q}C_t H_{\bar Q} + (1-\beta)H_{\bar Q} C_{t-1} H_{\bar Q} + \Big[\bar Q H \bar \pi (\beta \Sigma_t + (1-\beta) \Sigma_{t-1}) \bar \pi H \bar Q + \bar R_{C,t} - \tfrac{1}{\eta^2+\alpha_t+\bar R_{\alpha,t}}\bar b_t\bar b_t^T\Big] \Big)^{-1}.
     \]
     As the rescaled residual
     \[
        \Big[\bar Q H \bar \pi (\beta \Sigma_t + (1-\beta) \Sigma_{t-1}) \bar \pi H \bar Q + \bar R_{C,t} - \tfrac{1}{\eta^2+\alpha_t+\bar R_{\alpha,t}}\bar b_t\bar b_t^T\Big]/\sqrt{h_t} \leq C \sqrt{h_t} = C \alpha r^t
     \]
     is uniformly summable in $\alpha$, we can apply Propositions \ref{prop: autonomous} and \ref{prop: H reduction}, and Lemma \ref{lem: generic J} to verify $C_{t+1} = \Theta(\sqrt{h_t})$.

     Moving on to the off-diagonal, $- \frac{1}{\eta^2 + \alpha_t + \bar R_{\alpha, t} }  \bar b_t^T \bar S_t^{-1} = O(1)$, thus upon multiplying by $h_{t+1}$, this is $O(h_t)$, whereas the subsequent remainder terms $R_{b,t+1} = \Theta(\|\Sigma_{t+1}\|) = O(\sqrt{h_{t+1}})$. Thus, the contribution is negligible at the following iteration, and the uniform rate follows.

     We can immediately plug these rates into our risk bounds to yield $\sqrt{\det \Sigma_t} = \Theta(h_t^{(D+1)/4})$, $W(\mu_t) = O(h_t)$, thus yielding the desired rate.
\end{proof}

\subsubsection{Noisy Manifold}

We start off by iterating Corollary \ref{cor: shifted taylor} as described in the text. 
Let $\Sigma_{\operatorname{Null}^\perp} := \pi_{\operatorname{Null}^\perp} \Sigma \pi_{\operatorname{Null}^\perp}$, $\Sigma_{\operatorname{Null}} := \pi_{\operatorname{Null}} \Sigma \pi_{\operatorname{Null}}$.

\begin{lemma}[Weak Dependence]\label{lem: weak dependence}
    Suppose that for all $v \in \operatorname{Null}$, $\eta \in \operatorname{Null}^\perp$, $\eta^T \Sigma v = O(\|\Sigma_{\operatorname{Null}^\perp}\|)$, and $\Sigma_{\Null} = \Theta(1)$. There exists $H_{\Sigma_{\operatorname{Null}}}, T_{\Sigma_{\operatorname{Null}}}$ such that
    \[
    \mathcal{L}(\mu) = gg^T + H_{\Sigma_{\operatorname{Null}}} \Sigma_{\operatorname{Null}^\perp} H_{\Sigma_{\operatorname{Null}}} + g T_{\Sigma_{\operatorname{Null}}}(\Sigma_{\operatorname{Null}^\perp})^T + T_{\Sigma_{\operatorname{Null}}}(\Sigma_{\operatorname{Null}^\perp})g^T + R_{\Sigma_{\operatorname{Null}}}(\Sigma_{\operatorname{Null}^\perp}),
    \]
    $\|T_{\Sigma_{\operatorname{Null}}}(\Sigma_{\operatorname{Null}^\perp})\| = O(\|\Sigma_{\operatorname{Null}^\perp}\|),\ \|R_{\Sigma_{\operatorname{Null}}}(\Sigma_{\operatorname{Null}^\perp})\| = O(\|\Sigma_{\operatorname{Null}^\perp}\|^2)$.
\end{lemma}

\begin{proof}
    Applying our Gaussian ansatz,
    \[
    \Sigma_v := \operatorname{Cov}_{N(x^*, \Sigma)}( X | \pi_{\operatorname{Null}}[X - x^*] = v) = \Sigma_{\operatorname{Null}^\perp} - [\pi_{\operatorname{Null}^\perp}\Sigma \pi_{\operatorname{Null}} ]\Sigma_{\operatorname{Null}}^{-1} [\pi_{\operatorname{Null}}\Sigma \pi_{\operatorname{Null}^\perp}] = \Sigma_{\operatorname{Null}^\perp} + O(\|\Sigma_{\operatorname{Null}^\perp}\|^2).
    \]
    Thus, by Corollary \ref{cor: shifted taylor},
    \begin{align*}
    \mathcal{L}(\mu_v)    &= gg^T + H_v \Sigma_v H_v + g T_v(\Sigma_v)^T + T_v(\Sigma_v) g^T + R_v(\Sigma_v)\\
    &= gg^T + H_v \Sigma_{\operatorname{Null}^\perp} H_v + g T_v(\Sigma_{\operatorname{Null}^\perp})^T + T_v(\Sigma_{\operatorname{Null}^\perp}) g^T + R_v(\Sigma_{\operatorname{Null}^\perp}) + O(\|\Sigma_{\operatorname{Null}^\perp}\|^2),
    \end{align*}
    and this latter term can be absorbed into the remainder.
    Now, integrating out $v$ the relevant operators are
    \begin{align*}
        H_{\Sigma_{\operatorname{Null}}} = \int H_{\pi_{\operatorname{Null} [x - x^*]}}\, d\mu(x),\quad T_{\Sigma_{\operatorname{Null}}} = \int T_{\pi_{\operatorname{Null} [x - x^*]}}\, d\mu(x),\quad R_{\Sigma_{\operatorname{Null}}} = \int R_{\pi_{\operatorname{Null} [x - x^*]}}\, d\mu(x).
    \end{align*}
\end{proof}

While this matrix depends on the $\operatorname{Null}$ component of $\Sigma$, we will develop a recurrence where this is consistent across iterations. We now expand Lemma \ref{lem: schur} in the low-rank Hessian setting.

\begin{lemma}[Low-rank Schur]\label{lem: noisy schur}
For $\Sigma = \alpha I$ operators are defined as in Lemma \ref{lem: schur}, otherwise their equivalents as introduced in Lemma \ref{lem: weak dependence}.
 Decomposing in the $(\bar Q_{\operatorname{Null}^\perp}, \bar Q_{\operatorname{Null}}) := (\operatorname{col}(\overline{Q}) \cap \operatorname{Null}^\perp, \operatorname{col}(\bar Q) \cap \operatorname{Null})$ basis, let
 \begin{gather*}
     \tilde R_{C,t} := \bar R_{C,t} - \tfrac{1}{\eta^2+\alpha_t+\bar R_{\alpha,t}}\bar b_t\bar b_t^T,\quad \bar C_t := \beta C_t' + (1-\beta) C_t',\quad
     A_t = \bar Q_{\operatorname{Null}^\perp} \bar C_t \bar Q_{\operatorname{Null}^\perp} +  \bar Q_{\operatorname{Null}^\perp} \tilde R_{C, t} \bar Q_{\operatorname{Null}^\perp},\\
     W_t = \bar Q_{\operatorname{Null}} \tilde R_{C, t} \bar Q_{\operatorname{Null}} - \bar Q_{\operatorname{Null}} \tilde R_{C, t} \bar Q_{\operatorname{Null}^\perp} \bar [Q_{\operatorname{Null}^\perp} \bar C_t \bar Q_{\operatorname{Null}^\perp}]^{-1} \bar Q_{\operatorname{Null}^\perp} \tilde R_{C, t} \bar Q_{\operatorname{Null}}.
 \end{gather*}
\[
    C_{t+1} = h_{t+1}
    \begin{bmatrix}
        A_t^{-1} + A_t^{-1} [\bar Q_{\operatorname{Null}^\perp} \tilde R_{C, t} \bar Q_{\operatorname{Null}}] W_t^{-1} [\bar Q_{\operatorname{Null}} \tilde R_{C, t} \bar Q_{\operatorname{Null}^\perp}] A_t^{-1} & -A_t^{-1} \bar Q_{\operatorname{Null}^\perp} \tilde R_{C, t} \bar Q_{\operatorname{Null}} W_t^{-1}\\
        -W_t^{-1} Q_{\operatorname{Null}} \tilde R_{C, t} \bar Q_{\operatorname{Null}^\perp} A_t^{-1} & W_t^{-1}
    \end{bmatrix}
\]
\end{lemma}

\begin{proof}
    Recall that $C_{t+1} := t_{t+1}(\bar C_t  + \tilde R_{C,t})^{-1}$, hence this is the block Schur complement in the specified basis pivoted on the $\bar Q_{\operatorname{Null}^\perp} \times \bar Q_{\operatorname{Null}^\perp}$ block.
\end{proof}

\begin{corollary}[$\bar Q$ block]\label{cor: q block}
  Suppose that $\Sigma_0 = \alpha I$ and $t_1 = \alpha^2$. Then, in the $(\bar Q_{\operatorname{Null}^\perp}, \bar Q_{\operatorname{Null}})$  basis
  \begin{gather*}
  C_{1} =
  \begin{bmatrix}
     \alpha (\bar Q_{\operatorname{Null}^\perp}H^2 \bar Q_{\operatorname{Null}^\perp})^{-1} + O(\alpha^2) &  O(\alpha)\\
     O(\alpha) &  \alpha^2 [\bar Q_{\operatorname{Null}}\tilde R_{C,0} \bar Q_{\operatorname{Null}}]^{-1} +  O(\alpha)
  \end{bmatrix},\\
   C_{2} =
  \begin{bmatrix}
     \alpha H_{\bar Q_{\operatorname{Null}^\perp}}^{-1} H^2 H_{\bar Q_{\operatorname{Null}^\perp}}^{-1} + O(\alpha^2) &  O(\alpha)\\
     O(\alpha) &  \alpha^2 [\bar Q_{\operatorname{Null}}\tilde R_{C,1} \bar Q_{\operatorname{Null}}]^{-1} +  O(\alpha)
  \end{bmatrix}.
  \end{gather*}
  The matrix $\alpha^2 [\bar Q_{\operatorname{Null}}\tilde R_{C,t} \bar Q_{\operatorname{Null}}]^{-1} = \Theta(1)$, $t=0,1$.
\end{corollary}

\begin{proof}
    We can once again check this via direct computation utilizing Lemma \ref{lem: noisy schur}. As $\bar C_0$ comprises terms linear in $\Sigma_0$, it is $\Theta(\alpha)$, meanwhile the remainder terms $\tilde R_{C,0} = O(\alpha^2)$. Hence $A_0 = \bar Q_{\operatorname{Null}^\perp} \bar C_0 \bar Q_{\operatorname{Null}^\perp} + O(\alpha^2)$, $W_i = \bar Q_{\operatorname{Null}} \tilde R_{C, 0} \bar Q_{\operatorname{Null}} + O(\alpha^3)$, and the remaining approximations directly follow. As in the proof of Lemma \ref{cor: base case}, at the first step, the covariance is isotropic, hence in the $\bar Q_{\operatorname{Null}^\perp} \times \bar Q_{\operatorname{Null}^\perp}$ block the dominant terms is $\alpha (\bar Q_{\operatorname{Null}^\perp} H (I) H \bar Q_{\operatorname{Null}^\perp})^{-1} = \alpha (\bar Q_{\operatorname{Null}^\perp} H^2 \bar Q_{\operatorname{Null}^\perp})^{-1}$, however, after this step the gradient element of the covariance is a higher order residual, being of the Lemma \ref{lem: shift} constant shift form, and thus subsequently $\bar Q_{\operatorname{Null}^\perp} H \Sigma_1 H \bar Q_{\operatorname{Null}^\perp} = (\bar Q_{\operatorname{Null}^\perp} H \bar Q_{\operatorname{Null}^\perp}) \Sigma_1 (\bar Q_{\operatorname{Null}^\perp} H \bar Q_{\operatorname{Null}^\perp}) + O(\alpha^2) =: H_{\bar Q_{\operatorname{Null}^\perp}} \Sigma_1 H_{\bar Q_{\operatorname{Null}^\perp}} + O(\alpha^2).$
\end{proof}

\begin{corollary}[Noisy Manifold Base Case]\label{cor: noisy base case}
    Suppose that $X_0 = \alpha I$ and $h_1 = \alpha^2$. In the basis $(g, \overline{Q}_{\operatorname{Null}^\perp}, \overline{Q}_{\operatorname{Null}}),$ we have
    \begin{gather*}
    \Sigma_1 =
    \begin{bmatrix}
        \frac{\alpha^2}{\eta^2} + O(\alpha^3) & O(\alpha^2) & O(\alpha) \\
        O(\alpha^2) & \alpha (\bar Q_{\operatorname{Null}^\perp}H^2 \bar Q_{\operatorname{Null}^\perp})^{-1} + O(\alpha^2) & O(\alpha)\\
        O(\alpha) & O(\alpha) & \Theta(1)
    \end{bmatrix},\\
    \Sigma_2 =
    \begin{bmatrix}
        \frac{\alpha^2}{\eta^2} + O(\alpha^3) & O(\alpha^2) & O(\alpha) \\
        O(\alpha^2) & \alpha H_{\bar Q_{\operatorname{Null}^\perp}}^{-1} H^2 H_{\bar Q_{\operatorname{Null}^\perp}}^{-1} + O(\alpha^2) & O(\alpha)\\
        O(\alpha) & O(\alpha) & \Theta(1)
    \end{bmatrix}.
    \end{gather*}
\end{corollary}

\begin{proof}
    Corollary \ref{cor: q block} accounts for the bottom right $2\times 2$ block, what remains is to verify the top and left rows. Specifically, we must verify how
    \[
    C_1(b_0 + R_{b,0}) := \alpha^2 \left(\alpha H_{\bar Q}^2  + \bar R_{C,0} - \frac{1}{\eta^2 + \alpha_0 + \bar R_{\alpha, 0}}\bar b_0 \bar b_0^T\right)^{-1} (b_0 + R_{b,0})
    \]
    decomposes across $\bar Q_{\operatorname{Null}^\perp}, \bar Q_{\operatorname{Null}}$. Applying Corollary \ref{cor: q block}, we can express this block-wise as
    \[
    \begin{bmatrix}
        O(\alpha^2)\\
        O(\alpha)
    \end{bmatrix}
    =
    \begin{bmatrix}
        O(\alpha) & O(\alpha)\\
        O(\alpha) & O(1)
    \end{bmatrix}
    \begin{bmatrix}
        O(\alpha)\\
        O(\alpha)
    \end{bmatrix}
    \]
    yielding the desired rates. The proof for $\Sigma_2$ is essentially the same.
\end{proof}

Define $\mu(\zeta) = N(x^*, [\Sigma - \Sigma_{\operatorname{Null}^\perp}]+ \zeta \bar Q_{\operatorname{Null}})$, i.e. the distribution with $\operatorname{Null} \times \operatorname{Null}$ component fixed to $\zeta$.

\begin{lemma}[Filtered $\operatorname{Null}$]\label{lem: filtered}
        Assume that $D \geq 2d$ and $\operatorname{rank}(H) = 2d$. Let
        \[
        \Sigma_{t+1} = t_{t+1}[\beta \mathcal{L}(\mu_t(\zeta)) + (1-\beta) \mathcal{L}(\mu_{t-1}(\zeta)]^{-1}
        \]
        $\sqrt{t_{t}} = \alpha r^t$, $0<r<1,\ \beta > 0$. Then, for $\alpha, \zeta$ sufficiently small, $g^T \Sigma_t g = \Theta(h_t)$, for $v\in \operatorname{Null}$, $v^T \Sigma_t v = \Theta(1)$, and for $\eta \in \operatorname{Null}^\perp \cap g^\perp$, $\eta^T \Sigma \eta = \theta(\sqrt{h_t})$.
\end{lemma}

\begin{proof}
Applying Corollary \ref{cor: shifted taylor}, we can express
\[
\mathcal{L}(\mu) = \int \mathcal{L}(\mu_v)\, d\nu(v) = \int H_{v} \Sigma_v H_v  + gT_v(\Sigma_v)^T + T(\Sigma_v)^T + R_v(\Sigma_v)\, d\nu(v),
\]
where $\nu$ is the marginal density $\nu(v) = p_{\mu}(\pi_{\operatorname{Null}}[X - x^*] = v)$.

 We can apply Corollary \ref{cor: noisy base case} to verify that the desired anisotropy is present starting from an isotropic initialization. We first study the modified recurrence where the metric $M$ is in terms of $\mathcal{L}'(\mu)$ rather than the EGOP. We proceed by induction to verify that these rates are maintained. Our inductive hypothesis is that
     \begin{align*}
         \alpha_t = \Theta(h_t),\quad b_t = O(\sqrt{h_t}),\quad C_t = \Theta(\sqrt{h_t}), \quad \bar R_{C,t} = O(h_t),\\
         \bar R_{b,t} = O(\sqrt{h_t}),\quad \bar R_{\alpha,t} = O(\sqrt{h_t}),\quad  \| \bar Q_{\operatorname{Null}^\perp} \Sigma_t \bar Q_{\operatorname{Null}} \| = O(\sqrt{h_t}).
     \end{align*}
The weak $\operatorname{Null}$, $\operatorname{Null}^\perp$ correlation allows us to apply Lemma \ref{lem: weak dependence} to reduce the recurrence to the marginal form
\[
\mathcal{L}(\mu_{\zeta}) = gg^T + H_\zeta \Sigma_{\operatorname{Null}^\perp} H_{\zeta} + g T_\zeta(\Sigma_{\operatorname{Null}^\perp})^T + T_{\zeta}(\Sigma_{\operatorname{Null}^\perp})g^T + R_{\zeta}(\Sigma_{\operatorname{Null}^\perp}).
\]

By assumption $H_0 = H$ is maximal rank, hence for $\zeta$ sufficiently small this property is maintained for $H_\zeta$, and we restrict our analysis to this neighborhood.
The verification for the $g\times g$ block is exactly the same as verified in Theorem \ref{thm: Generic Rate} as it is of constant shift form (see Lemma \ref{lem: shift}), thus we move on to the remaining terms. Explicit computations are similar to those previously presented in Corollaries \ref{cor: q block} and \ref{cor: noisy base case}. First, we immediately have that
\[
A_t = H_{\bar Q_{\operatorname{Null}^\perp}} [\beta C_{t} + (1-\beta) C_{t-1}] H_{\bar Q_{\operatorname{Null}^\perp}} + O(h_t),\quad W_t = \bar Q_{\operatorname{Null}} \tilde R_{C, t} \bar Q_{\operatorname{Null}} + O(h_t^{3/2}).
\]
Thus for the middle block, applying Lemma \ref{lem: noisy schur} yields
\[
\bar Q_{\operatorname{Null}^\perp} C_{t+1} \bar Q_{\operatorname{Null}^\perp} = t_{t+1} \{H_{\bar Q_{\operatorname{Null}^\perp}} [\beta (\bar Q_{\operatorname{Null}^\perp} C_{t} \bar Q_{\operatorname{Null}^\perp}) + (1-\beta) (\bar Q_{\operatorname{Null}^\perp} C_{t-1} \bar Q_{\operatorname{Null}^\perp})] H_{\bar Q_{\operatorname{Null}^\perp}} + O(h_t)\}^{-1}.
\]
by the reductions from Propostions \ref{prop: autonomous} and \ref{prop: H reduction}, we can reduce this to an autonomous recurrence in $Q_{\operatorname{Null}^\perp} C_{t} \bar Q_{\operatorname{Null}^\perp}$ that converges at the desired rate by Lemma \ref{lem: generic J}. See the the proof of Theorem \ref{thm: Generic Rate} for more details.

Applying Lemma \ref{lem: noisy schur} again yields
\[
\bar Q_{\operatorname{Null}} C_{t+1} \bar Q_{\operatorname{Null}} = h_{t+1} W_t^{-1} = h_{t+1}(\bar Q_{\operatorname{Null}} \tilde R_{C, t} \bar Q_{\operatorname{Null}})^{-1} + O(\sqrt{h_t}).
\]
As $\bar Q_{\operatorname{Null}} \tilde R_{C, t} \bar Q_{\operatorname{Null}} = O(\|\Sigma_{\operatorname{Null}^\perp}\|^2) = O(h_t)$, it follows that $h_{t+1}(\bar Q_{\operatorname{Null}} \tilde R_{C, t} \bar Q_{\operatorname{Null}})^{-1} = \Omega(1)$, and thus it is of constant order by our covariance cap. Thus this rate is an immediate consequence of the second order anisotropy, and the orthogonal component and other remainders only contribute negligibly, justifying the uniform rate. The remaining terms in the Lemma \ref{lem: noisy schur} block decomposition are
\[
\bar Q_{\operatorname{Null}^\perp} \Sigma_{t+1} \bar Q_{\operatorname{Null}}  = \bar Q_{\operatorname{Null}^\perp} C_{t+1} \bar Q_{\operatorname{Null}} = - h_{t+1} A_t^{-1} (\bar Q_{\operatorname{Null}^\perp} \tilde R_{C, t} \bar Q_{\operatorname{Null}}) W_t^{-1},\quad  b_{t+1} = u^T \Sigma_{t+1} \bar Q =  C_{t+1}(b_t + R_{b,t}),
\]
for $u = g/\|g\|$ the normalized gradient vector. In the first term, the product $(\bar Q_{\operatorname{Null}^\perp} \tilde R_{C, t} \bar Q_{\operatorname{Null}}) W_t^{-1} = \Theta(1)$, hence the rate is dominated by the second order anisotropy $h_{t+1} A_t^{-1} = \Theta(\sqrt{h_t})$. For the final vector, $\|b_t + R_{b,t}\| = O(\|\Sigma_{\operatorname{Null}^\perp}\|)$, hence we can combine this with our previous analysis to yield
\[
\bar Q_{\operatorname{Null}} C_{t+1}(b_t + R_{b,t}) = \Theta(1) O(\sqrt{h_t})) = O(\sqrt{h_{t+1}}),\quad  \bar Q_{\operatorname{Null}^\perp} C_{t+1}(b_t + R_{b,t}) = \Theta(h_{t+1}) O(\sqrt{h_t}) = O(h_{t+1}).
\]
Note that the scalar inflation to go from $h_t$ to $h_{t+1}$ does not aggregate across iterations as the initial bound was driven by the second order anisotropy term, to which the remainders only contribute negligibly.
\end{proof}

This verifies the claim for the recurrence in $\mathcal{L}(\mu_{\zeta})$. Our goal now is to lift this analysis to $\mathcal{L}(\mu)$ with covariance cap $\zeta$. The distinction between these recurrences is that the first a priori filters the $\operatorname{Null}$ direction out of the dynamics, whereas the latter trims any eigenvalues above a set threshold $\zeta$. These recurrences coincide if the $\operatorname{Null} \times \operatorname{Null}$ component never falls below the threshold $\zeta$, thus we seek to identify such a threshold.

\begin{proof}[Proof of Theorem \ref{thm: manifold rate}]
    Lemma \ref{lem: filtered} verifies the specified dynamics, particularly the $\operatorname{Null}\times \operatorname{Null}$ component of $\Sigma_{i,\zeta}$ remains $\Theta(1)$, and thus there is some floor $\gamma(\zeta)$ which it never drops below. We now will argue that $\gamma(\zeta)>0$ as $\zeta \to 0$, hence there exists a threshold $\zeta^*$ such that the filtered and cap dynamics coincide.

    As $\zeta \to 0$, $H_{\zeta} \to H$, as the distribution concentrates about $x^*$, and in particular, for $\zeta$ small enough, the signature of $H_\zeta$ is the same as that of $H$. This follows from the maximal rank assumption on $H$. By Propositions \ref{prop: autonomous} and \ref{prop: H reduction}, upon transforming $C_{t,\zeta} \mapsto \frac{W_{\zeta}^T C_{t,\zeta} W_{\zeta}}{\sqrt{c(r,\beta) h_t}} =: X_{t,\zeta}$, $W_\zeta := U_{\zeta} |\Lambda_{\zeta}|^{1/2}$, for the diagonalization $H_{\zeta} = U_{\zeta} \Lambda_\zeta U_\zeta^T$, these recurrences are reduced to the same canonical form, with seeds depending smoothly on $\zeta$. Applying Lemma \ref{lem: non-expansive}, it follows that $X_{t,\zeta} = X_{t,0} + O(\zeta + \alpha)$.

    Inverting the transformation, this yields $C_{t, \zeta} = C_{t,0} + O(\sqrt{h_t}(\delta + \alpha)).$ Returning to the proof of Lemma \ref{lem: filtered},
    \[
    \Sigma_{t,\operatorname{Null}, \zeta} = h_{t+1}(\bar Q_{\operatorname{Null}} \tilde R_{C, t, \zeta}(\Sigma_{t, \operatorname{Null}^\perp, \zeta}) \bar Q_{\operatorname{Null}})^{-1} + O(\sqrt{h_t}) = \Sigma_{t, \operatorname{Null}^\perp, 0} + O(\sqrt{h_t}(\zeta+ \alpha)),
    \]
    hence for $\zeta, \alpha$ sufficiently small the noise floor is non-vanishing as desired.

    Thus the rates of Lemma \ref{lem: filtered} is achieved. Plugging these rates into our risk bounds yields $\sqrt{\det \Sigma_t} = \Theta(h_t^{(2d+1)/4})$, $W(\mu_t) = O(h_t)$, thus yielding the desired rate.
\end{proof}

\section{EGOP Approximation}\label{app: numerical}

In this section, we show that metric anisotropy not only improves point regression, but accelerates the estimation of the EGOP matrix in our two model settings.
We argue that the smoothed gradients, while potentially biased estimators at individual points, in the aggregate form a matrix that differs from the population quantity by a benign perturbation, leading to an equivalent convergence analysis at adequate sample sizes. 

Define the smoothed gradient and EGOP as
\[
\nabla_\Sigma f(x) = \underset{v}{\operatorname{argmin}} \min_a \int \|f(z) - a - v(z-x)\|^2\, dN(x, \Sigma),\quad \mathcal{L}_\Sigma(\mu) = \int \nabla_\Sigma f \nabla_\Sigma f^T\, d\mu.
\]
We study the covariance recurrence
\begin{align}\label{eq: smooth EGOP}
    \Sigma_{t+1} = h_{t+1} [\beta \hat{\mathcal{L}}_{\Sigma_t}(\mu_t) + (1-\beta) \hat{\mathcal{L}}_{\Sigma_{t-1}}(\mu_{t-1})]^{-1} = : h_{t+1} [\hat M_t]^{-1}.
\end{align}

\begin{lemma}[Smoothed Gradient]\label{lem: smoothed grad}
There exists $T_{x}, C$, $T_{x}(\Sigma) \leq C \|\Sigma\|$ such that, for $\mu = N(x^*, \Sigma),$
\[
\nabla_\Sigma f = \nabla f + T_{x}(\Sigma),\quad \mathcal{L}_\Sigma(\mu) = \mathcal{L}(\mu) + \mathbb{E}_{\mu}[gT_X(\Sigma)^T + T_X(\Sigma) g^T] + O(\|\Sigma\|^{2}).
\]
\end{lemma}

\begin{proof}
    The standard OLS solution yields
    \[
    \nabla_{\Sigma} f(x) = \Sigma^{-1} \operatorname{Cov}_{N(x,\Sigma)}(X, f(X)) = \mathbb{E}_{N(x,\Sigma)}[\nabla f(X)],
    \]
    with the last equality following from Stein's Identity \citep{stein1981estimation}. Applying a Taylor expansion, we see that 
    \[
    \mathbb{E}_{N(x,\Sigma)}[\nabla f(X)] = \int \nabla f(x) + H(y-x) + T_x(y-x,y-x)\, dN(0, \Sigma) =: \nabla f(x) + T_x(\Sigma),
    \]
    where $T_x$ is a second order remainder. For the second claim, the proof follows by a simple expansion,
    \begin{align*}
    \mathcal{L}_\Sigma(\mu) &= \int \nabla_\Sigma f \nabla_\Sigma f^T\, d\mu\\
    &= \int \nabla f \nabla f^T\, d\mu + \int T_x(\Sigma) \nabla f^T + \nabla f T_x(\Sigma)\, d\mu\\
    &= \mathcal{L}(\mu) + \int T_x(\Sigma) g^T + g T_x(\Sigma)\, d\mu + + \int T_x(\Sigma) (\nabla f-g)^T + (\nabla f-g) T_x(\Sigma)\, d\mu,
    \end{align*}
    satisfying the claim, with full details being similar to those in the proof of Lemma \ref{lem: Generic Taylor}.
\end{proof}

This expansion is satisfactory for the full-rank setting, indeed, one can substitute $\mathcal{L}_\Sigma$ into the recurrence of Equation \ref{eq: oracle egop}, and arrive at the same expansion of Lemma \ref{lem: Generic Taylor}. A more refined expansion is necessary for the continuous index setting. Let $\operatorname{Null}_x$ denote the null subspace at point $x$.

\begin{lemma}(Anisotropic Smoothed Gradient)\label{lem: aniso smoothed grad}
    Let $(X,Y)$ satisfy the supervised noisy manifold hypothesis, $D>2d$. There exists $T_{x,\Sigma_{\operatorname{Null}_x}}$ and a uniform constant $C$ satisfying $T_{x,\Sigma_{\operatorname{Null}_x}}(A) \leq C\|A\|$,  such that
    \[
    \nabla_{\Sigma} f(x) = \nabla f(x) + T_{x, \Sigma_{\Null_x}}(\Sigma_{\Null_x^\perp} - [\pi_{\operatorname{Null}_x^\perp}\Sigma \pi_{\operatorname{Null}_x} ]\Sigma_{\operatorname{Null}_x}^{-1} [\pi_{\operatorname{Null}_x}\Sigma \pi_{\operatorname{Null}_x^\perp}])
    \]
\end{lemma}

\begin{proof}
    We combine Stein's identity with Lemma \ref{cor: shifted taylor} to yield
    \begin{align*}
        \nabla_\Sigma f(x) = \mathbb{E}_{N(x,\Sigma)}[\nabla f(X)] &= \int \nabla f(y)\, d N(y;x, \Sigma)\\
        &= \int \left(\int \nabla f(y)\, d\mu_v\right)\, d\nu(v) = \nabla f(x) + \int T_v(\Sigma_v) d\nu(v)\\
        &=: \nabla f(x) + T_{x, \Sigma_{\Null_x}}(\Sigma_{\Null_x^\perp} - [\pi_{\operatorname{Null}_x^\perp}\Sigma \pi_{\operatorname{Null}_x} ]\Sigma_{\operatorname{Null}_x}^{-1} [\pi_{\operatorname{Null}_x}\Sigma \pi_{\operatorname{Null}_x^\perp}])  .
    \end{align*}    
\end{proof}

\begin{lemma}[Anisotropic Rotation]\label{lem: aniso rot}
    For $x$ within a neighborhoood of $x^*$,
    \begin{gather*}
    \|\pi_{\operatorname{Null}_x^\perp}\Sigma \pi_{\operatorname{Null}_x}\| = O(\|\pi_{\Null^\perp_{x^*}}[x-x^*]\| +  \|\pi_{\Null^\perp} \Sigma \pi_{\Null}\|),\quad   \|\Sigma_{\Null^\perp_x}^{-1}\|  \leq 2/\lambda_{\min} (\Sigma_{\Null^\perp}),\\
    \|\Sigma_{\Null^\perp_x}\| = O(\|\pi_{\Null^\perp_{x^*}}[x-x^*]\|^2 +  \|\Sigma_{\Null^\perp}\| + \|\pi_{\Null^\perp_{x^*}}[x-x^*]\| \|\Sigma \pi_{\Null^\perp} \|).
    \end{gather*}
\end{lemma}

\begin{proof}
    We first show that
    \[
    \|\pi_{\Null_x} - \pi_{\Null_{x^*}}\| = O(\pi_{\Null_{x^*}^\perp}[x-x^*]),
    \]
    from which the desired result follows easily. As characterized in the proof of Lemma \ref{lem: grad form}, $\Null_x = \phi_{\pi(x),\nabla f(x)}(\mathcal{T}_x)^\perp$, for $\phi_{\pi(x),u}(w):= \mathrm{I\!I}_{\pi(x)}(u, w)$. Thus we see that this subspace changes smoothly in the displacement of the basepoint $\pi(x)$ or the gradient $\nabla f(x)$, i.e. along $\Null^\perp_x$.

    Thus, setting $D_x := \pi_{\Null_{x}^\perp} - \pi_{\Null_{x^*}^\perp}$,
    \begin{align*}
    \|\pi_{\operatorname{Null}_x^\perp}\Sigma \pi_{\operatorname{Null}_x}\| &= \|(\pi_{\Null_{x^*}^\perp} + D_x) \Sigma (\pi_{\Null_{x^*}} - D_x)\|\\
    & \leq \|\pi_{\Null^\perp} \Sigma \pi_{\Null}\| + C \|D_x\| + C'\|D_x\|^2\\
    &= O(\|\pi_{\Null^\perp} \Sigma \pi_{\Null}\| + \|\pi_{\Null^\perp_{x^*}}[x-x^*]\|),
    \end{align*}
    as desired. The third claim follows in the exact same way, stopping at the second line of the derivation. The second claim follows by continuity.
\end{proof}

It is again necessary to localize and extend $f$ so that these constraints are globally active, or contract the covariance cap $\zeta_t \to 0$ at a logarithmic rate.

\begin{lemma}[Anisotropic Gradient Bound]\label{lem: aniso grad bound}
    Let $(X,Y)$ satisfy the supervised noisy manifold hypothesis, $D>2d$. There exists $T_{x,\Sigma_{\operatorname{Null}_x}}$ and a uniform constant $C$ satisfying $T_{x,\Sigma_{\operatorname{Null}_x}}(A) \leq C\|A\|$,  such that
    \[
    \nabla_{\Sigma} f(x) = \nabla f(x) + T_{x, \Sigma_{\Null_x}}(\Sigma_{\Null^\perp}) + R(x),
    \]
    for $R(x)= O(\|\pi_{\Null^\perp}[x - x^*]\|^2 + \|\pi_{\Null^\perp}[x - x^*]\|\|\pi_{\Null^\perp} \Sigma\| + \|\pi_{\Null^\perp} \Sigma\|^2 + \|\Sigma_{\Null^\perp}\|)$.
\end{lemma}

\begin{proof}
    By Lemma \ref{lem: aniso smoothed grad} we can express
    \begin{align*}
    \nabla_{\Sigma} f(x) &= \nabla f(x) + T_{x, \Sigma_{\Null_x}}(\Sigma_{\Null_x^\perp} - [\pi_{\operatorname{Null}_x^\perp}\Sigma \pi_{\operatorname{Null}_x} ]\Sigma_{\operatorname{Null}_x}^{-1} [\pi_{\operatorname{Null}_x}\Sigma \pi_{\operatorname{Null}_x^\perp}])\\
    &=:\nabla f(x) + T_{x, \Sigma_{\Null_x}}(\Sigma_{\Null^\perp}) + R(x).
    \end{align*}
    That $R(x)$ satisfies the appropriate bound is an immediate consequence Lemma \ref{lem: aniso rot} and the Lipschitz continuity of $T$.
\end{proof}

\begin{lemma}[Anisotropic Smoothed EGOP]\label{lem: aniso egop}
    Let $(X,Y)$ satisfy the supervised noisy manifold hypothesis, $D>2d$. Then
    \[
        \mathcal{L}_\Sigma(\mu) = \mathcal{L}(\mu) + \mathbb{E}_\mu[g T_{X, \Sigma_{\Null_{X}}}(\Sigma_{\Null^\perp})^T + T_{X, \Sigma_{\Null_{X}}}(\Sigma_{\Null^\perp})g^T ] + R(\Sigma),
    \]
    for 
    \[
    R(\Sigma) = O(\|\Sigma_{\Null^\perp}\| + \|\pi_{\Null^\perp} \Sigma\|^2 + \sqrt{\|\Sigma_{\Null^\perp}\|} \|\pi_{\Null^\perp} \Sigma\|).
    \]
\end{lemma}

\begin{proof}
    This result is an immediate consequence of Lemma \ref{lem: aniso grad bound} in combination with Isserlis's Theorem \cite{isserlis1918formula}, see Lemma \ref{lem: Generic Taylor} for a similar computation.
\end{proof}

Thus we have achieved a remainder control that is tight under anisotropy, where $\Sigma$ may stretch in the $\Null$ direction without severe consequence. Our primary assumption will be regarding the accuracy of gradient estimation.

\begin{lemma}[EGOP Approximation]\label{lem: egop approx}
    Suppose that $\|\nabla_{\Sigma} \hat f - \nabla_{\Sigma} f\|_{L^2(\mu)} = O(n^{-1}\det(\Sigma_t^{-1/2}))$. Then, 
    \[
    \hat{\mathcal{L}}_{\Sigma}(\mu) = \mathcal{L}_{\Sigma}(\mu) + O\left(\frac{1}{n\sqrt{\det(\Sigma_t)}}\right).
    \]
\end{lemma}
\begin{proof}
    We compute,
    \begin{align*}
        \|\hat{\mathcal{L}}_{\Sigma}(\mu) - \mathcal{L}_{\Sigma}(\mu)\| &= \left\|\int \nabla_{\Sigma} \hat f \nabla_{\Sigma} \hat f^T - \nabla_{\Sigma} f \nabla_{\Sigma} f^T\, d\mu \right\|\\
        &\leq \int \|\nabla_{\Sigma} \hat f - \nabla_{\Sigma} f\|\|\nabla_{\Sigma} f\| + \|\nabla_{\Sigma} \hat f - \nabla_{\Sigma} f\|\|\nabla_{\Sigma} \hat f\|\, d\mu\\
        & = O\left(\frac{1}{n\sqrt{\det(\Sigma_t)}}\right).
    \end{align*}
\end{proof}

We note that the indicated rate is typical when utilizing standard nonparametric estimators, such as via local linear regression, and can be achieved by taking the covariance cap $\zeta \to 0$ at an appropriate rate. This condition is less easily satisfied for fully anisotropic $\Sigma_t$.

\begin{proposition}[Full Rank Empirical Guarantee]\label{prop: Full Rank Emp}
    Let $H$ be full rank, and suppose that $\|\nabla_{\Sigma_t} \hat f - \nabla_{\Sigma_t} f\|_{L^2(\mu_t)} \leq C n^{-1}\det(\Sigma_t^{-1/2})$. Adopt the assumptions of Theorem \ref{thm: Generic Rate}. Then, for $\Sigma_t$ satisfying Equation \ref{eq: smooth EGOP}, for $t_n = \frac{4 \log n}{(D+5) \log(1/r)}$, 
    \[
         \mathbb{E}\left[ \int (\hat{P}_{\hat M_{t_n}}(f) - f)^2\, d\mu_{t_n} \right] = O\left(n^{-\frac{4}{D+5}}\right).
    \]
    
\end{proposition}

\begin{proof}
    By Lemma \ref{lem: smoothed grad} we see that $\hat{\mathcal{L}}(\mu)$ satisfies an equivalent expansion to Lemma \ref{lem: Generic Taylor} up to $O(n^{-1}\det(\Sigma_t^{-1/2}))$ approximation error. As argued in the proof of Theorem \ref{thm: Generic Rate}, the second order recurrence is asymptotically autonomous so long as the empirical error is negligible, satisfying $n^{-1}\det(\Sigma_t^{-1/2}) = O(h_t)$. Evaluating,
    \[
    n^{-1}\det(\Sigma_{t_n}^{-1/2}) = n^{-1} h_{t_n}^{-(d+1)/4} = O(h_{t_n}).
    \]
\end{proof}

\begin{proposition}[Continuous Index Empirical Guarantee]\label{prop: cont ind emp}
    Let $(X,Y)$ satisfy the supervised noisy manifold hypothesis, and  $\operatorname{rank}(H) = 2d$. Suppose that $\|\nabla_{\Sigma_t} \hat f - \nabla_{\Sigma_t} f\|_{L^2(\mu_t)} \leq C n^{-1}\det(\Sigma_t^{-1/2})$. Adopt the assumptions of Theorem \ref{thm: manifold rate}. Then, for the recurrence
    \[
    \Sigma_t = h_{t+1} [\beta \hat{\mathcal{L}}_{\Sigma}(\mu_t) + (1-\beta) \hat{\mathcal{L}}_{\Sigma}(\mu_{t-1})]^{-1},
    \]
    selecting $t_n = \frac{4 \log n}{(2d+5) \log(1/r)}$,
    \[
         \mathbb{E} \left[ \int (\hat{P}_{M_{t_n}}(f) - f)^2\, d\mu_{t_n} \right] = O\left(n^{-\frac{4}{2d+5}}\right).
    \]
    
\end{proposition}

\begin{proof}
    The proof again follows immediately from our previous analysis, exploiting the robustness to perturbations we already verified. First however, we must verify that the expansions of Lemma \ref{lem: aniso egop} and Lemma \ref{lem: weak dependence} are equivalent, i.e. we must verify that 
    \begin{align}\label{eq: aniso condition}
     \|\pi_{\Null^\perp} \Sigma_t\|^2, \sqrt{\|\Sigma_{t,\Null^\perp}\|} \|\pi_{\Null^\perp} \Sigma_t\| = O(\|\Sigma_{t,\Null^\perp}\|),
    \end{align}
    along the smoothed recurrence. We do this first neglecting noise. The base case is the same, and we adopt the inductive hypothesis in the proof of Theorem \ref{thm: manifold rate}. 
    Condition \eqref{eq: aniso condition} is weaker than this inductive hypothesis, and the remainder of the argument is identical.
    Thus, it suffices to check that empirical error is negligible, satisfying $n^{-1}\det(\Sigma_t^{-1/2}) = O(h_t)$. Evaluating,
    \[
    n^{-1}\det(\Sigma_{t_n}^{-1/2}) = n^{-1} h_{t_n}^{-(2d+1)/4} = O(h_{t_n}).
    \]
\end{proof}

\subsection{Additional Settings of Interest}
\begin{example}[Spheres]
        Built into Theorem \ref{thm: manifold rate} are assumptions on the ambient dimension and derivatives of $f$. The sphere $S^{D-1}$ is a typical toy example, however $D\geq 2(D-1)$ fails for $D\geq 3$. In these settings, the hessian has the form
        \[
        H =
        \begin{bmatrix}
            T & b\\
            b^T & 0
        \end{bmatrix}
        \]
        in the $\mathcal{T} \times \mathcal{N}$ basis. Thus the noisy manifold hypothesis only guarantees that a single entry of the Hessian is 0. Hence, $\pi_{g^\perp} H \pi_{g^\perp}$ is generically full rank, driving second order anisotropy along the tangent and orthogonal, as demonstrated in Figure \ref{fig: sphere}.
\end{example}

    \begin{figure}[ht]
        \centering
        \includegraphics[width=0.5\linewidth]{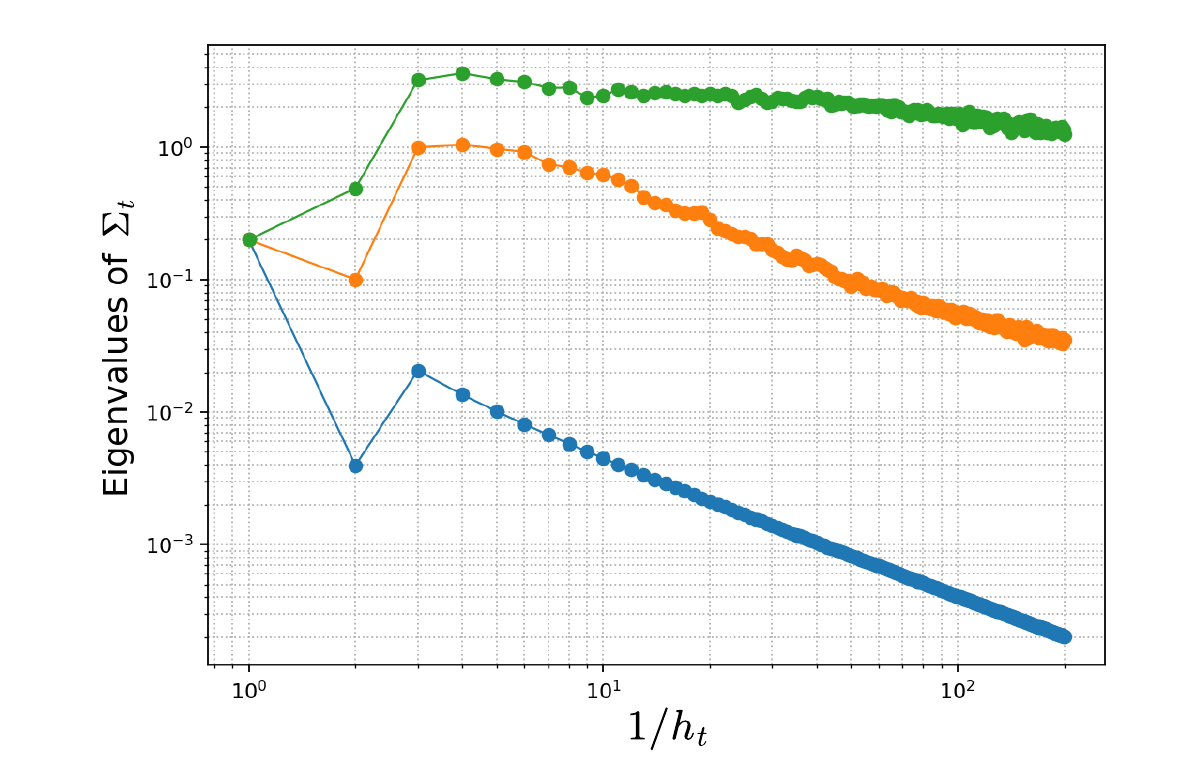}
        \caption{Eigenvalue decay of $\Sigma_t$ with Local EGOP Learning applied to data satisfying the noisy manifold hypothesis about $S^2$. The orthogonal exhibits light decay, approaching second order anisotropy asymptotically.}
        \label{fig: sphere}
    \end{figure}

\end{document}